\documentclass{article}

\usepackage[preprint]{neurips_2026}

\usepackage[utf8]{inputenc} 
\usepackage[T1]{fontenc}    
\usepackage{hyperref}       
\usepackage{url}            
\usepackage{booktabs}       
\usepackage{amsfonts}       
\usepackage{nicefrac}       
\usepackage{microtype}      
\usepackage{xcolor}         
\usepackage{amsmath,amssymb}
\usepackage{graphicx}
\usepackage{float}
\usepackage{placeins}
\usepackage{amsthm}
\usepackage{booktabs}
\usepackage{subcaption}

\usepackage{amsthm}

\usepackage{subcaption}
\newtheorem{theorem}{Theorem}[section]

\newtheorem{proposition}[theorem]{Proposition}
\title{FrameVGGT: Coherence-Preserving Memory for Bounded Streaming Geometry}

\author{
  Zhisong Xu \\
  Institute of Industrial Science \\
  The University of Tokyo \\
  Tokyo, Japan \\
  \texttt{zhisongxv@cvl.iis.u-tokyo.ac.jp}
  \And
  Takeshi Oishi \\
  Institute of Industrial Science \\
  The University of Tokyo \\
  Tokyo, Japan \\
  \texttt{oishi@cvl.iis.u-tokyo.ac.jp}
}

\begin{document}

\maketitle

\begin{abstract}
Streaming Visual Geometry Transformers such as StreamVGGT enable strong online 3D perception, but their KV-cache grows unbounded over long streams, limiting practical deployment. We study bounded-memory streaming geometry from the perspective of memory organization: unlike language modeling, where useful information can often be compressed at token level, geometry-driven inference relies on coherent and mutually compatible observations across views. Under fixed memory budgets, retaining history as isolated entries can progressively fragment the geometric context needed for stable long-horizon matching and fusion. We therefore propose \textbf{FrameVGGT}, a bounded-memory framework that maintains a fixed-capacity set of complementary memory units for streaming geometry. In our implementation, each unit is instantiated as a frame-wise KV segment summarized by a compact key-space prototype, together with a sparse anchor tier for persistent long-range references. Across long-sequence 3D reconstruction, video depth estimation, and camera pose estimation, FrameVGGT achieves favorable accuracy--memory trade-offs under bounded budgets while maintaining more stable geometry over long streams.
\end{abstract}

\section{Introduction}

Dense 3D reconstruction from images is fundamental to robotic perception, navigation, and interaction~\cite{sheng2024review,barsan2018robust,turan2018sparse,mascaro2025scene,raychaudhuri2025semantic}.
While classical SfM/MVS pipelines achieve high fidelity through explicit geometric optimization, they are typically multi-stage and computationally expensive~\cite{andersen1998perception,saputra2018visual,hussain2021comprehensive}.
Recent feed-forward geometry models instead learn to infer cameras and dense geometry directly from multi-view observations~\cite{Wang_2024_CVPR,duisterhof2025mast3r,wang2025vggt,wang2025pi,lin2025depth,keetha2025mapanything}, but extending them to online long-horizon streams remains challenging because preserving all past context causes memory and latency to grow with sequence length.

Long-horizon reconstruction has been addressed in part by chunk-based or SLAM-style systems.
These methods control sequence length by dividing the stream into chunks or submaps, selecting keyframes, and aligning local reconstructions through optimization or geometric backends~\cite{deng2025vggt,maggio2025vggt,liu2025slam3r}.
While effective, such pipelines often introduce additional alignment stages, hand-designed keyframe policies, or iterative optimization, which weakens the simplicity of feed-forward inference.

A more end-to-end streaming alternative is to process frames causally while maintaining a bounded state or cache.
Implicit-state methods compress history into a latent representation, but this can weaken long-range constraints and induce drift~\cite{wang2025continuous,chen2025ttt3r}.
Explicit-memory methods instead cache past representations for reuse, but naive accumulation grows unbounded and therefore requires eviction or selection~\cite{zhuo2025streaming}.
Under bounded resources, recent explicit-memory methods often manage retained history at token granularity using importance, diversity, or eviction heuristics~\cite{liu2023scissorhands,zhang2023h2o,li2024snapkv,yuan2026infinitevggt,su2026xstreamvggt,mahdi2025evict3r,lu2026ovggt}.

\begin{figure}[t]
  \centering
  \includegraphics[width=\linewidth,page=1]{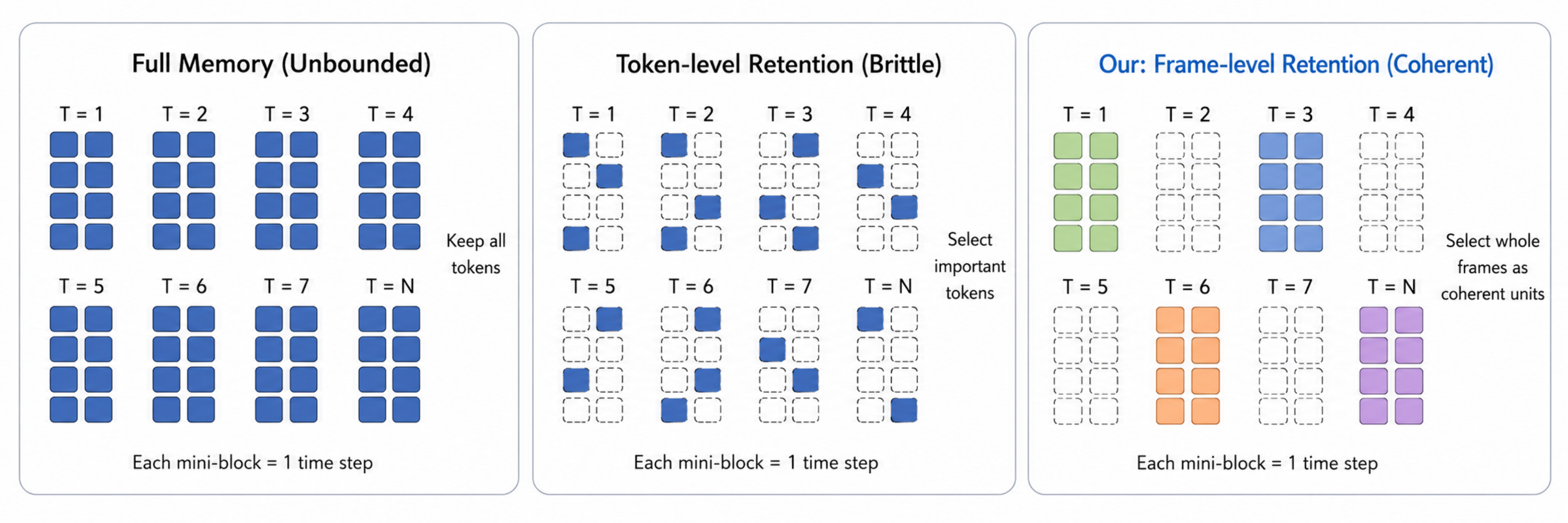}
  \caption{Memory organization strategies for bounded streaming geometry.
  Full memory grows unbounded over time, whereas token-level retention can fragment geometric context under fixed budgets.
  Frame-level retention preserves coherent frame-wise cache segments for downstream inference.}
  \label{fig:citecompare}
\end{figure}
However, a fixed memory budget is insufficient without proper cache management. As shown in Fig.~\ref{fig:citecompare}, cache selection granularity matters critically.
Token-level selection treats tokens independently. Because geometric tokens are noisy and unevenly informative, this leads to fragmented retention: each frame is represented by only a sparse subset of tokens, destroying structural coherence.
By contrast, frame-level retention keeps all tokens from a selected frame together as a coherent unit, preserving both its internal structure and cross-view compatibility.

This matters for geometry because, unlike language modeling, multi-view inference often relies on mutually compatible frame-level evidence rather than isolated tokens.
Once such evidence is fragmented, the cache may remain populated but provide weaker geometric support.
Thus, retention granularity and frame-level selection become key design axes for bounded geometric transformers.

We instantiate this perspective with \textbf{FrameVGGT}, an inference-time bounded-memory framework that retains history as coherent frame-wise units.
FrameVGGT maintains a fixed-capacity set of complementary memory units and adds a lightweight \textbf{sparse anchor tier} for persistent long-range references.
Together, they preserve structured multi-view context under a fixed memory budget, enabling stable long-horizon streaming geometry.
\noindent\textbf{Contributions.}

(1) We formulate \textbf{retention granularity} as a key design axis for bounded streaming geometry, showing that coherent geometric context matters beyond cache size alone.

(2) We provide diagnostics showing that token-level retention can induce context thinning, spatio-temporal fragmentation, and brittle fusion under weak redundancy.

(3) We introduce \textbf{FrameVGGT}, an inference-time two-tier memory framework that preserves frame-wise KV segments as coherent units and uses sparse anchors for long-range consistency under fixed memory budgets.

\section{Related Work}

\subsection{Geometry-based Reconstruction}
Classical 3D reconstruction relies on explicit geometric optimization. Offline pipelines such as SfM and MVS jointly estimate camera poses and scene structure over all input views~\cite{agarwal2011building,schonberger2016structure,furukawa2015multi,zhu2024slm}.
These approaches demonstrate promising accuracy and robustness, but they require full-scene access and computationally expensive global optimization. 
Online systems such as SLAM update motion and maps incrementally~\cite{mur2015orb,taketomi2017visual}. However, because they typically focus on sparse or semi-dense representations, dense long-horizon reconstruction remains challenging under strict memory budgets.

\subsection{Learning-based Reconstruction with Fixed-size Inputs}
Recent learning-based approaches perform feed-forward geometric inference and jointly predict camera parameters, depth, and correspondences across multiple views. 
However, these models assume fixed-size inputs, so memory usage and computational cost increase rapidly as the number of views grows. 
Several extensions improve scalability through sub-map decomposition, anchor-based representations~\cite{deng2025sail}, or token-level acceleration, but remain largely batch-oriented.

\subsection{Learning-based Reconstruction under Streaming Constraints}
Streaming methods process frames online while maintaining history. Explicit-memory approaches cache past features, tokens, or states for reuse~\cite{wang20253d,wu2025point3r,chen2025long3r}.
However, naive accumulation leads to growing memory and latency. 
Implicit-state approaches compress history into bounded latent states, which improves efficiency at the cost of weaker long-range constraints and increased drift. 
Windowed methods restrict computation to local temporal neighborhoods~\cite{li2025wint3r}, which stabilizes runtime but limits long-horizon consistency.
These methods mainly differ in how history is stored or compressed under streaming constraints. By contrast, we focus on memory organization under bounded streaming, emphasizing retention granularity and coherent multi-view context.


\section{Problem Definition and Analysis}
\label{sec:problem}

We consider online geometric inference over an unbounded image stream \(I=\{I_t\}_{t\ge1}\), where the framework outputs \(y_t\) sequentially. 
For each decoder layer \(l\) in the stream backbone (see Fig.~\ref{fig:streamvggt_memory_pipeline}), let \(\tilde{C}_t^{(l)}\) denote the retained Transformer cache at time \(t\), let \(C_{t,\mathrm{new}}^{(l)}\) denote the newly produced cache entries, and let \(\Omega(\cdot)\) denote a bounded-memory cache update operator. 
To maintain cache efficiency, we move from the unbounded cache update \(\tilde{C}_t^{(l)} = \tilde{C}_{t-1}^{(l)} \oplus C_{t,\mathrm{new}}^{(l)}\)
to a bounded-memory streaming regime:
\[
\tilde{C}_t^{(l)} =
\Omega \!\left(\tilde{C}_{t-1}^{(l)} \oplus C_{t,\mathrm{new}}^{(l)}\right),
\qquad
\text{s.t.}\quad
\sum_l |\tilde{C}_t^{(l)}| \le M,
\]
where \(|\cdot|\) denotes cache size and \(M \in \mathbb{N}\) is the maximum cache size.

We define geometric context as observations that jointly satisfy multi-view consistency. In this framework, cross-attention acts as differentiable feature matching—akin to SfM—rather than sparse semantic retrieval. Unlike language models that prioritize individual salient tokens, geometric inference relies on redundant, spatially distributed, and mutually compatible evidence. Consequently, token utility hinges on their retention as coherent groups, as only collective observations can sustain the matching and long-range consistency essential for geometric reconstruction.

We refer to this requirement as \emph{geometric coherence}.

Under bounded memory, token-level retention may degrade the retained context in several structurally important ways:
\begin{itemize}
    \item \textbf{Context thinning:} a fixed budget is spread over more frames, reducing the retained fraction per frame and weakening geometric evidence.
    \item \textbf{Frame-level fragmentation:} each frame may be represented by too few tokens to preserve within-frame structure or cross-view compatibility.
    \item \textbf{Directional concentration:} retained keys may align with a narrow local mode, biasing retrieval toward a small subset of memory.
\end{itemize}
These failure modes reduce correspondence diversity and weaken long-horizon consistency, especially when memory is controlled at overly fine granularity.

\paragraph{Directional Concentration Diagnostic.}
Under bounded memory, attention-based memory retrieval can become overly concentrated on a small temporally local subset of retained keys. 
To analyze this effect, we introduce a directional contrast statistic \(\Delta_k\) that measures how strongly a local subset is directionally separated from the rest of memory.


For a query step \(k\), let \(R_k\) denote a local subset of retained entries and \(N_k\) its complement, with \(R_k \cap N_k = \varnothing\).
Let \(\hat{\mathbf{k}}_i \in \mathbb{R}^{d_l}\) denote an \(\ell_2\)-normalized key vector at layer \(l\), where \(d_l\) is the key dimension. 
We define \(\boldsymbol{\mu}_R\) as the normalized mean direction of the local subset \(R_k\), as follows:
\[
\boldsymbol{\mu}_R =
\frac{\frac{1}{|R_k|}\sum_{i\in R_k}\hat{\mathbf{k}}_i}
{\left\|\frac{1}{|R_k|}\sum_{i\in R_k}\hat{\mathbf{k}}_i\right\|_2}.
\]
Let \(A_R\) denote the average directional alignment of the local subset \(R_k\) with respect to \(\boldsymbol{\mu}_R\), and \(A_N\) that of the non-local subset \(N_k\).
The directional contrast is defined as
\[
\Delta_k = A_R - A_N
=
\frac{1}{|R_k|}\sum_{i\in R_k}\cos(\hat{\mathbf{k}}_i,\boldsymbol{\mu}_R)
-
\frac{1}{|N_k|}\sum_{j\in N_k}\cos(\hat{\mathbf{k}}_j,\boldsymbol{\mu}_R).
\]
%
This statistic is relevant because retrieval is determined by the attention weights \(\alpha_i\).
Here, \(\alpha_i=\exp(s_i)/\sum_j \exp(s_j)\), and the pre-softmax similarity score is \(s_i=\mathbf{q}^\top \hat{\mathbf{k}}_i\), where \(\mathbf{q}\in\mathbb{R}^{d_l}\) denotes the query vector. 

When the query \(\mathbf{q}\) is positively aligned with \(\boldsymbol{\mu}_R\), a larger \(\Delta_k\) indicates that entries in \(R_k\) have a stronger average directional advantage than those in \(N_k\).
As a result, entries in \(R_k\) tend to receive higher similarity scores on average.
The softmax then amplifies this gap and concentrates more attention on the local subset.
This shifts probability mass away from non-local memory and reduces the effective context used for retrieval and fusion.
A formal treatment under a simplified two-group attention model is deferred to the appendix.


\paragraph{Why Frame-Level Retention?}
Retention granularity directly affects the stability of geometric inference under bounded memory.
If retained key groups preserve broader directional variation, directional concentration becomes weaker.
This reduces the risk of collapse into a narrow local mode.

When the local set \(R_k\) is partitioned into source-frame key segment \(\{S_t\}\), variance decomposition yields
\[
\Delta_k
\le
1-\frac{1}{2}\sum_t w_t \sigma_t^2 - A_N,
\qquad
w_t=\frac{|S_t|}{|R_k|},
\]
where \(S_t\) contains the retained key vectors originating from frame segment \(t\),
\(\bar{\mathbf{k}}_t\) is their group mean, and \(\sigma_t^2\) denotes their within-group directional dispersion around \(\bar{\mathbf{k}}_t\) (Appendix).
The set-level direction \(\boldsymbol{\mu}_R\) remains the common reference for \(A_R\), \(A_N\), and \(\Delta_k\).
Thus, larger within-group dispersion suppresses \(\Delta_k\), reducing the local subset's directional advantage over non-local entries and making retrieval less likely to collapse into a narrow local mode.

This explains why retention granularity matters.
Token-level retention can repeatedly keep only a few highly aligned keys from the same local mode, reducing dispersion, increasing \(A_R\), and strengthening concentration.
Frame-level retention instead keeps source-frame key groups together, preserving coherent geometric context and broader directional spread.
Under bounded memory, this acts as a geometric regularizer for stable multi-view matching and fusion over long horizons.

\begin{figure}[t]
  \centering
  \includegraphics[width=0.9\linewidth]{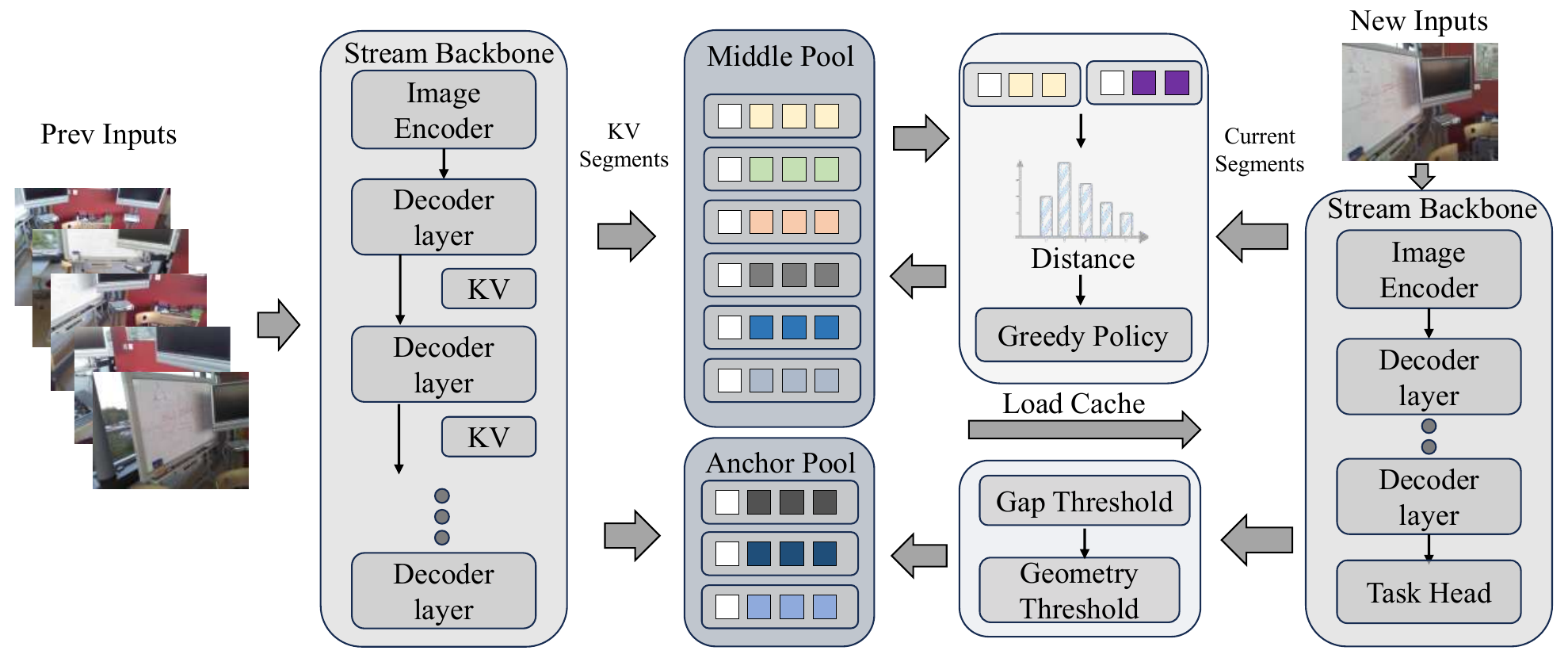} %
  \caption{Pipeline of FrameVGGT. Previous inputs are encoded into per-layer KV blocks, which are managed by a mid-term pool with a distance-based greedy selection policy and an anchor pool gated by gap and geometry thresholds. The selected cache is then loaded to condition new inputs for streaming inference.}
  \label{fig:streamvggt_memory_pipeline}
\end{figure}

\section{Method}
\label{sec:method}
We formulate bounded streaming geometry as a memory-constrained problem.
The key requirement is to preserve retrievable multi-view information that supports geometric matching and fusion.
FrameVGGT addresses this problem with a two-tier memory design based on a mid-term memory bank and a lightweight sparse anchor memory.
Figure~\ref{fig:streamvggt_memory_pipeline} shows the pipeline of our framework. 
The method consists of three components: 1) retrieval-space scoring via segment dissimilarity, 2) a mid-term memory bank for complementary coverage, and 3) a sparse anchor memory for long-range persistence.

\paragraph{Retrieval-Space Scoring via Segment Dissimilarity.}
The analysis in Sec.~3 suggests that retention should favor complementary memory units over near-duplicate ones, as this increases intra-group dispersion $\sigma_t^2$ and mitigates directional concentration.
Accordingly, we define a retrieval-space scoring scheme for frame-wise memory units based on segment dissimilarity.

We organize memory into frame-wise memory units. 
For each time step \(t\) and layer \(l\), let \(S_t^{(l)}\) denote the KV segment produced by frame \(t\). Let \(H_l\) be the number of attention heads, \(d_l\) the per-head key dimension, and \(\mathcal{T}_t\) the token index set introduced at frame \(t\). 
The segment keys are \(\{\mathbf{k}_{t,h,\tau}^{(l)}\in\mathbb{R}^{d_l}\}_{h=1,\tau\in\mathcal{T}_t}^{H_l}\). 
To enable efficient inter-segment comparison, we compute the mean key of each segment in retrieval space,
\begin{equation}
\mathbf{v}_t^{(l)}
=
\frac{1}{H_l\,|\mathcal{T}_t|}
\sum_{h=1}^{H_l}\sum_{\tau\in\mathcal{T}_t}\mathbf{k}_{t,h,\tau}^{(l)},
\end{equation}
followed by \(\ell_2\) normalization,
\[
\bar{\mathbf{v}}_t^{(l)}=\mathbf{v}_t^{(l)}/\|\mathbf{v}_t^{(l)}\|_2.
\]
Rather than explicitly modeling spatial topology, this mean key serves only as a lightweight retrieval-space proxy for segment selection. It enables online estimation of redundancy and complementarity across segments, while the selected segment itself is retained with its full KV entries.

We measure inter-segment dissimilarity by cosine distance of the mean keys,
\begin{equation}
d\!\left(S_i^{(l)},S_j^{(l)}\right)
=
1-\left\langle \bar{\mathbf{v}}_i^{(l)},\bar{\mathbf{v}}_j^{(l)}\right\rangle.
\end{equation}
Because memory access is governed by key--query compatibility, we perform selection directly in retrieval space. Large \(\Delta_k\) reflects concentration around a narrow local mode, whereas our scoring scheme suppresses near-duplicate segments and preserves under-covered directions, thereby limiting \(\Delta_k\) and maintaining broad retrievable support.


\paragraph{Mid-Term Memory Bank.}
The key role of the mid-term bank is to preserve broad coverage in retrieval space under a fixed memory budget.
Rather than keeping many redundant nearby segments, it retains a small set of complementary segments that covers under-represented regions.

Formally, given a candidate segment set \(\mathcal{M}_t\), the mid-term bank seeks a subset \(Q \subseteq \mathcal{M}_t\) with \(|Q| \le M_{\mathrm{mid}}\), where \(M_{\mathrm{mid}}\) is the bank capacity, that minimizes the worst-case coverage gap,
\begin{equation}
Q^\ast
=
\arg\min_{Q\subseteq \mathcal{M}_t}
\max_{S\in \mathcal{M}_t}\min_{S'\in Q} d(S,S').
\end{equation}
This objective ensures that every candidate segment remains close to at least one retained segment.
As a result, it discourages redundant selections from dense local clusters and instead favors complementary segments that improve overall coverage.

To optimize this objective, we approximate it online using greedy farthest-first selection initialized from the most recent segment.
For each candidate segment \(S\in\mathcal{M}_t\), we maintain its nearest-set distance
\begin{equation}
m(S)=\min_{S'\in Q} d(S,S').
\end{equation}
Larger \(m(S)\) indicates that \(S\) lies in an under-covered region of retrieval space.
Starting from the most recent segment, we iteratively augment \(Q\) with the candidate of largest \(m(S)\).
Let \(S_{\mathrm{new}}\) denote the newly selected segment.
We then update
\begin{equation}
m(S)\leftarrow \min\!\bigl(m(S),\, d(S,S_{\mathrm{new}})\bigr), \qquad \forall\, S\in\mathcal{M}_t .
\end{equation}
This procedure provides an efficient online approximation that maintains a broad and non-redundant retrieval basis under a fixed budget.

At each step, the current frame produces a per-layer frame segment \(S_t^{(l)}\), which is appended to the candidate pool.
When the pool exceeds capacity, the retained subset is refreshed by the above rule and reassembled into the bounded KV memory for the next step.
Selection is performed independently for each layer.

\paragraph{Sparse Anchor Memory.}
The mid-term bank provides the primary bounded working memory, but local context may degrade over very long streams.
We therefore maintain a sparse anchor bank \(A_t\), with \(|A_t|\le B_A\), to preserve long-range references without excessive local redundancy.
The first frame is always retained as a persistent global reference and is excluded from replacement.

For a candidate frame \(i\), anchor promotion is governed by temporal spacing, frame reliability, and novelty:
\[
A_t=
\begin{cases}
\operatorname{FIFO}_{B_A}\!\left(A_{t-1}\cup\{i\}\right),
& \text{if } \Delta_i\ge G,\ \Phi(i)\ge\tau_{\mathrm{conf}},\ \nu(i)\ge\tau_{\mathrm{novel}},\\[4pt]
A_{t-1},
& \text{otherwise.}
\end{cases}
\]
Here, \(\Delta_i=i-i_{\mathrm{last}}\) measures the gap from the last promoted anchor, and \(G\) is the minimum spacing.
The reliability score \(\Phi(i)=q_i s_i\) combines normalized prediction confidence \(q_i\) and Laplacian sharpness \(s_i\), while \(\nu(i)\) measures the minimum distance to existing anchor pose signatures.
The thresholds \(\tau_{\mathrm{conf}}\) and \(\tau_{\mathrm{novel}}\) are fixed across datasets.
FIFO is applied only to historical anchors, so the first-frame reference is never evicted.
This compact tier extends long-range coverage beyond the mid-term bank.

\section{Experiments}

\subsection{Experimental Setup}
\label{sec:exp_setup}
\begin{table}[t]
\centering
\footnotesize
\setlength{\tabcolsep}{3pt}
\renewcommand{\arraystretch}{1.08}
\caption{Reconstruction results on 7-Scenes and NRGBD. Best results are shown in bold.}
\label{tab:recon_main}
\resizebox{\linewidth}{!}{%
\begin{tabular}{l cccccc|cccccc}
\toprule
{Method}
& \multicolumn{6}{c|}{7-Scenes}
& \multicolumn{6}{c}{NRGBD}
\\
\cmidrule(lr){2-7}\cmidrule(lr){8-13}
& Acc$\downarrow$ & Acc$_{\text{med}}\downarrow$ & Comp$\downarrow$ & Comp$_{\text{med}}\downarrow$ & NC$\uparrow$ & NC$_{\text{med}}\uparrow$
& Acc$\downarrow$ & Acc$_{\text{med}}\downarrow$ & Comp$\downarrow$ & Comp$_{\text{med}}\downarrow$ & NC$\uparrow$ & NC$_{\text{med}}\uparrow$
\\
\midrule
CUT3R        & 0.181 & 0.127 & 0.095 & 0.0327 & 0.525 & 0.531 & 0.322 & 0.237 & 0.128 & 0.0334 & 0.553 & 0.566 \\
Point3R      & 0.063 & 0.026 & 0.031 & 0.0150 & 0.559 & 0.589 & 0.113 & 0.048 & 0.037 & 0.0060 & 0.621 & 0.696 \\
TTT3R        & 0.060 & 0.035 & 0.028 & 0.0050 & 0.560 & 0.588 & 0.161 & 0.082 & 0.093 & 0.0147 & 0.602 & 0.654 \\
XStreamVGGT  & 0.105 & 0.052 & 0.048 & 0.0100 & 0.556 & 0.583 & 0.128 & 0.080 & 0.043 & 0.0069 & 0.621 & 0.703\\
InfiniteVGGT & 0.041 & 0.016 & 0.024 & 0.0047 & 0.561 & 0.593 & 0.078 & 0.049 & 0.035 & 0.0069 & 0.647 & 0.757 \\
OVGGT*       & 0.033 & 0.011 & 0.020 & 0.0040 & 0.560 & 0.591 & 0.056 & 0.036 & 0.032 & 0.0069 & 0.635 & 0.731 \\
Ours(12)     & 0.035 & 0.013 & 0.020 & 0.0048 & 0.562 & 0.594 & 0.054 & 0.036 & 0.031 & 0.0068 & 0.656 & 0.764 \\
Ours(16)     & 0.033 & 0.010 & 0.019 & 0.0044 & \textbf{0.564} & 0.597 & \textbf{0.053} & \textbf{0.035} & \textbf{0.030} & 0.0061 & 0.663 & 0.773 \\
Ours(20)     & \textbf{0.027} & \textbf{0.008} & \textbf{0.018} & \textbf{0.0038} & 0.563 & 0.595 & \textbf{0.053} & 0.036 & 0.033 & 0.0067 & 0.661 & 0.776 \\
Ours(24)     & 0.028 & 0.009 & 0.019 & 0.0040 & \textbf{0.564} & \textbf{0.598} & 0.054 & \textbf{0.035} & \textbf{0.030} & \textbf{0.0060} & \textbf{0.670} & \textbf{0.782} \\
\bottomrule
\end{tabular}%
}
\end{table}

We evaluate FrameVGGT on \emph{online 3D reconstruction}, \emph{video depth estimation}, and \emph{monocular camera pose estimation} under a strict \emph{resource-controlled protocol}: pretrained weights, frame sampling, and inference pipeline are fixed, while only the bounded-memory organization varies under comparable KV-cache budgets.
Baselines marked with \(*\) indicate contemporaneous token-level streaming methods.
FrameVGGT is used purely at inference time without retraining or fine-tuning, and memory denotes the effective streaming KV-cache footprint.
All experiments use a single NVIDIA RTX A6000 GPU.

\paragraph{Memory-budget sweep.}
We vary the mid-term capacity $M$, i.e., the maximum number of refreshable frame segments retained in the mid-term bank, over $M\in\{12,16,20,24\}$.
This tests whether preserving broader complementary coverage improves long-horizon performance, and whether gains saturate once the dominant memory structure is covered.

\paragraph{Timescale-separation diagnostic.}
To isolate the role of sparse persistent references, we compare \textbf{24 mid-term + 0 anchor} against \textbf{20 mid-term + 4 anchor} under the same total memory budget.
This tests whether long-range references provide benefits beyond the primary mid-term memory bank.

\paragraph{Short-term continuity vs.\ complementary context.}
We evaluate a \textbf{Recent-$K$} variant that reserves capacity for the last $K$ frames ($K\in\{2,4,6\}$) and allocates the rest to mid-term memory.
This tests whether bounded memory should prioritize recent continuity or complementary context under a fixed budget.

\begin{figure}[t]
    \centering
    \includegraphics[width=\linewidth,page=1]{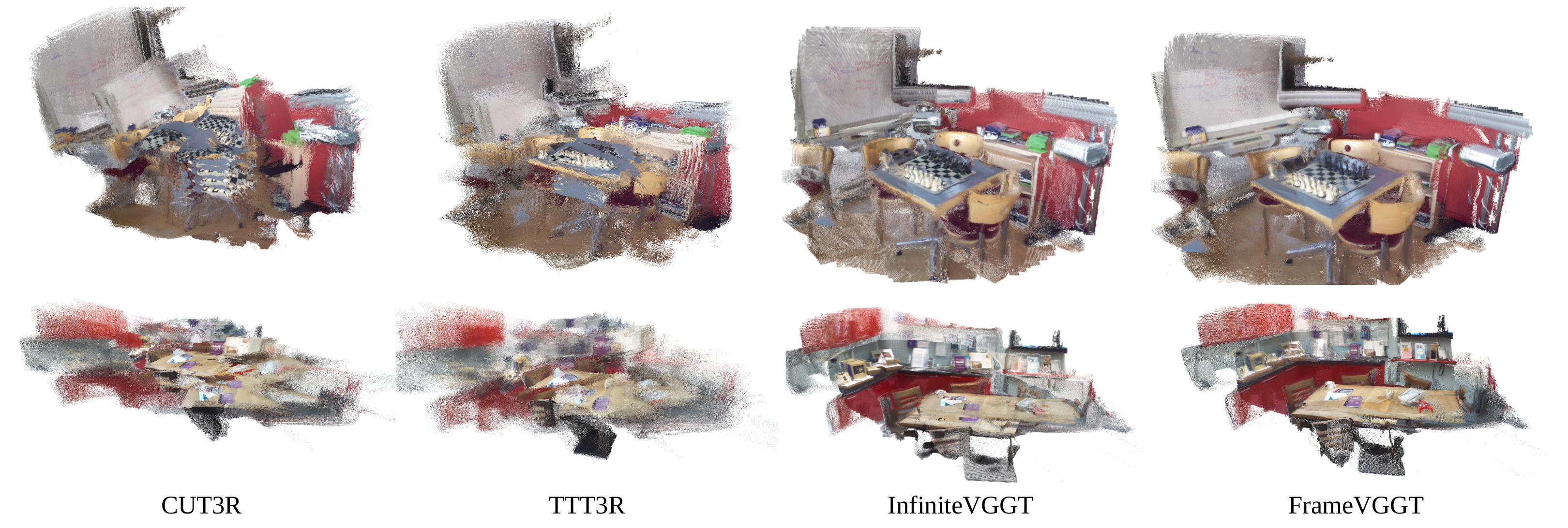}
    \caption{Reconstruction visualization comparison on the 7scenes dataset. Our method ensures robust geometric consistency over extended sequences, maintaining a highly stable and reliable result.}
    \label{fig:recon_vis}
\end{figure}

\begin{table}[t]
\centering
\footnotesize
\setlength{\tabcolsep}{4pt}
\renewcommand{\arraystretch}{1.1}
\caption{Reconstruction results on Long3D. Best results are shown in bold.}
\label{tab:long3d_main}
\begin{tabular}{lcccccc}
\toprule
{Method}
& Acc$\downarrow$ & Acc$_{\text{med}}\downarrow$ & Comp$\downarrow$ & Comp$_{\text{med}}\downarrow$ & NC$\uparrow$ & NC$_{\text{med}}\uparrow$
\\
\midrule
CUT3R        & 3.217 & 2.235 & 1.605 & 1.351 & 0.500 & 0.491 \\
TTT3R        & 2.964 & 2.310 & 2.265 & 1.805 & 0.507 & 0.514 \\
InfiniteVGGT & 2.138 & 1.728 & \textbf{0.963} & \textbf{0.364} & 0.513 & 0.526 \\
OVGGT*       & 2.148 & 1.687 & 1.594 & 0.574 & 0.526 & 0.530 \\
Ours         & \textbf{1.698} & \textbf{1.331} & 1.032 & 0.433 & \textbf{0.536} & \textbf{0.540} \\
\bottomrule
\end{tabular}
\end{table}

\begin{table*}[t]
\centering
\scriptsize

\setlength{\tabcolsep}{4pt}
\renewcommand{\arraystretch}{1.1}
\captionof{table}{Performance comparison across anchor configurations. Best results are shown in bold.}
\label{tab:anchor_ablation_cases}
\resizebox{\linewidth}{!}{%
\begin{tabular}{l cccccc|cccccc}
\toprule
{Scene}
& \multicolumn{6}{c|}{Anchor=0}
& \multicolumn{6}{c}{Anchor=4}
\\
\cmidrule(lr){2-7}\cmidrule(lr){8-13}
& Acc$\downarrow$ & Acc$_{\text{med}}\downarrow$ & Comp$\downarrow$ & Comp$_{\text{med}}\downarrow$ & NC$\uparrow$ & NC$_{\text{med}}\uparrow$
& Acc$\downarrow$ & Acc$_{\text{med}}\downarrow$ & Comp$\downarrow$ & Comp$_{\text{med}}\downarrow$ & NC$\uparrow$ & NC$_{\text{med}}\uparrow$
\\
\midrule
office/seq-09
& 0.0429 & \textbf{0.0088} & 0.0126 & \textbf{0.0044} & 0.5627 & 0.6085
& \textbf{0.0275} & \textbf{0.0088} & \textbf{0.0119} & \textbf{0.0044} & \textbf{0.5758} & \textbf{0.6163}
\\
stairs/seq-01
& 0.1392 & 0.0556 & \textbf{0.0915} & \textbf{0.0077} & 0.5411 & 0.5585
& \textbf{0.1250} & \textbf{0.0521} & 0.0932 & 0.0081 & \textbf{0.5565} & \textbf{0.5726}
\\
\bottomrule
\end{tabular}%
}

\vspace{0.8em}

\setlength{\tabcolsep}{4pt}
\renewcommand{\arraystretch}{1.1}
\captionof{table}{Recent ablation on 7-Scenes and NRGBD under a fixed total budget ($M_{\text{total}}{=}16$). Recent-$0$ corresponds to Ours ($M{=}16$). Best results are shown in bold.}
\label{tab:recon_recent_sweep}
\resizebox{\linewidth}{!}{%
\begin{tabular}{l cccccc|cccccc}
\toprule
{Method}
& \multicolumn{6}{c|}{7-Scenes}
& \multicolumn{6}{c}{NRGBD}
\\
\cmidrule(lr){2-7}\cmidrule(lr){8-13}
& Acc$\downarrow$ & Acc$_{\text{med}}\downarrow$ & Comp$\downarrow$ & Comp$_{\text{med}}\downarrow$ & NC$\uparrow$ & NC$_{\text{med}}\uparrow$
& Acc$\downarrow$ & Acc$_{\text{med}}\downarrow$ & Comp$\downarrow$ & Comp$_{\text{med}}\downarrow$ & NC$\uparrow$ & NC$_{\text{med}}\uparrow$
\\
\midrule
Recent-0 & \textbf{0.033} & \textbf{0.010} & \textbf{0.019} & \textbf{0.0044} & \textbf{0.564} & \textbf{0.597} & \textbf{0.053} & \textbf{0.035} & \textbf{0.030} & \textbf{0.0061} & \textbf{0.663} & \textbf{0.773} \\
Recent-2 & 0.037 & 0.013 & 0.020 & \textbf{0.0044} & 0.562 & 0.593 & 0.056 & 0.040 & 0.031 & 0.0070 & 0.656 & 0.768 \\
Recent-4 & 0.053 & 0.017 & 0.024 & 0.0050 & 0.558 & 0.587 & 0.066 & 0.045 & 0.031 & 0.0073 & 0.637 & 0.734 \\
Recent-6 & 0.069 & 0.025 & 0.027 & 0.0050 & 0.555 & 0.583 & 0.085 & 0.059 & 0.034 & 0.0080 & 0.623 & 0.707 \\
\bottomrule
\end{tabular}%
}

\vspace{1em}

\begin{minipage}[t]{0.48\linewidth}
\centering
\scriptsize
\setlength{\tabcolsep}{6pt}
\renewcommand{\arraystretch}{1.1}
\captionof{table}{Video depth estimation results on Bonn. Best results are shown in bold.}
\label{tab:depth_main}
\begin{tabular}{lcc}
\toprule
Method & Abs Rel$\downarrow$ & $\delta<1.25\uparrow$ \\
\midrule
CUT3R        & 0.0831 & 0.9402 \\
Point3R      & 0.0809 & 0.9513 \\
TTT3R        & 0.0752 & 0.9587 \\
XStreamVGGT  & 0.0807 & 0.9697 \\
InfiniteVGGT & 0.0560 & \textbf{0.9801} \\
OVGGT*       & 0.0658 & 0.9558 \\
Ours (12)    & 0.0526 & 0.9793 \\
Ours (16)    & 0.0525 & \textbf{0.9801} \\
Ours (20)    & 0.0514 & 0.9799 \\
Ours (24)    & \textbf{0.0512} & 0.9799 \\
\bottomrule
\end{tabular}
\end{minipage}\hfill
\begin{minipage}[t]{0.48\linewidth}
\centering
\scriptsize
\setlength{\tabcolsep}{7pt}
\renewcommand{\arraystretch}{1.1}
\captionof{table}{Recent ablation on Bonn under a fixed total budget ($M_{\text{total}}{=}16$). Recent-$0$ corresponds to Ours ($M{=}16$). Best results are shown in bold.}
\label{tab:depth_slidewindow}
\begin{tabular}{lcc}
\toprule
Sliding Window & Abs Rel$\downarrow$ & $\delta<1.25\uparrow$ \\
\midrule
Recent-$0$ & \textbf{0.0525}  & \textbf{0.9801} \\
Recent-$2$ & 0.0535& 0.9794\\
Recent-$4$ & 0.0562& 0.9788\\
Recent-$6$ & 0.0589& 0.9780 \\
\bottomrule
\end{tabular}
\end{minipage}

\label{tab:appendix_results_block}
\end{table*}

\subsection{3D Reconstruction}

\paragraph{Protocol and metrics.}
We evaluate 3D reconstruction on \textsc{7-Scenes}, \textsc{NRGBD}~\cite{blanton2020extending,azinovic2022neural}, and \textsc{Long3D}, reporting Accuracy (Acc, $\downarrow$), Completeness (Comp, $\downarrow$), and Normal Consistency (NC, $\uparrow$). For \textsc{7-Scenes} and \textsc{NRGBD}, we test long-horizon sequences (up to 1000 frames, stride 2) under a unified bounded-memory protocol. For \textsc{Long3D} (2k--10k frames), predicted point clouds are downsampled to 2\% to align with ground-truth density.

\paragraph{Effect of mid-term capacity.}
Tab.~\ref{tab:recon_main}, \ref{tab:long3d_main} and Fig.~\ref{fig:recon_vis} demonstrate that token-level methods degrade on long sequences due to context fragmentation, which reduces surface coverage and introduces artifacts. In contrast, our coherence-aligned memory preserves multi-view context, yielding superior accuracy--memory trade-offs. By operating on frame-level prototypes, \textsc{FrameVGGT} bypasses the overhead of fine-grained token sampling, achieving significantly higher throughput than baselines. For instance, our $1.9$--$3.7$\,GB footprint (12--24 blocks) is substantially more efficient than \textsc{InfiniteVGGT} (6.9\,GB) or \textsc{XStreamVGGT} (10.3\,GB).

\paragraph{Global-anchor diagnostic.}
Tab.~\ref{tab:anchor_ablation_cases} shows that anchors enhance robustness during unreliable mid-term context (e.g., blur, occlusion, or weak parallax) by providing persistent references. While their impact is minimal in regular scenarios, anchors mitigate error accumulation in degraded cases. Anchor-only baselines (FIFO, random, uniform) consistently underperform, confirming that anchors serve as a vital robustness-oriented complement to mid-term memory.

\paragraph{Recent vs.\ mid-term evidence.}
Tab.~\ref{tab:recon_recent_sweep} shows that forcing a Recent-$K$ buffer often degrades reconstruction quality.
Recency-biased buffering over-allocates memory to highly overlapping adjacent frames while displacing more complementary mid-term memory, reducing effective scene coverage and weakening the support needed to maintain surface completeness and geometric consistency over long horizons.

\subsection{Video Depth Estimation}

\paragraph{Protocol and metrics.}
We evaluate streaming video depth estimation on \textsc{Bonn}~\cite{palazzolo2019refusion} over sequences of up to 500 frames, reporting Abs Rel ($\downarrow$) and $\delta<1.25$ ($\uparrow$) after per-sequence scale alignment.
The evaluation protocol is fixed across all memory settings.

\paragraph{Effect of mid-term capacity.}
Tab.~\ref{tab:depth_main} shows that our frame-level memory maintains strong depth accuracy even under tight bounded-memory settings.
Compared with reconstruction and pose, video depth is less sensitive to memory size, suggesting that much of the required support remains relatively local.
Still, increasing the mid-term capacity $M$ yields modest but consistent improvements in Abs Rel, indicating that additional mid-term context remains beneficial when local evidence becomes ambiguous or insufficient.
The overall trend also suggests an earlier saturation regime: once nearby context is already strong, additional history provides diminishing returns.

\paragraph{Recent vs.\ mid-term evidence.}
Tab.~\ref{tab:depth_slidewindow} shows that reserving memory for the most recent frames does not improve depth estimation, and performance gradually declines as $K$ increases. This suggests that, as in 3D reconstruction, bounded memory is more effective when used to preserve complementary context rather than redundant adjacent views.

\subsection{Camera Pose Estimation}
\label{sec:exp_pose}

\paragraph{Protocol and metrics.}
We evaluate streaming camera pose estimation on \textsc{TUM-dynamics}~\cite{6385773} using sequences of up to 300 frames.
We report ATE, RPE${}_{\text{trans}}$, and RPE${}_{\text{rot}}$ (all $\downarrow$) after Sim(3) Umeyama alignment, used only for global similarity normalization.

\paragraph{Effect of mid-term capacity.}
Similar trends are observed for pose estimation.
Tab.~\ref{tab:pose_main} and Fig.~\ref{fig:tum_pose_vis} show that token-level streaming suffers from accumulated error and drift as sequences grow. In contrast, our frame-level memory preserves complementary multi-view context over time, resulting in more stable pose estimation and lower ATE and RPE over long sequences. Increasing the mid-term capacity $M$ consistently improves pose accuracy, with diminishing returns at larger budgets.

\paragraph{Recent vs.\ mid-term evidence.}
Tab.~\ref{tab:pose_slidewindow} shows that forcing a Recent-$K$ buffer is insufficient for long-horizon pose stability. Consistent with 3D reconstruction, allocating more memory to recent frames reduces capacity for complementary mid-term context and weakens long-range geometric constraints.

\begin{figure}[t]
    \centering
    \includegraphics[width=\linewidth,page=1]{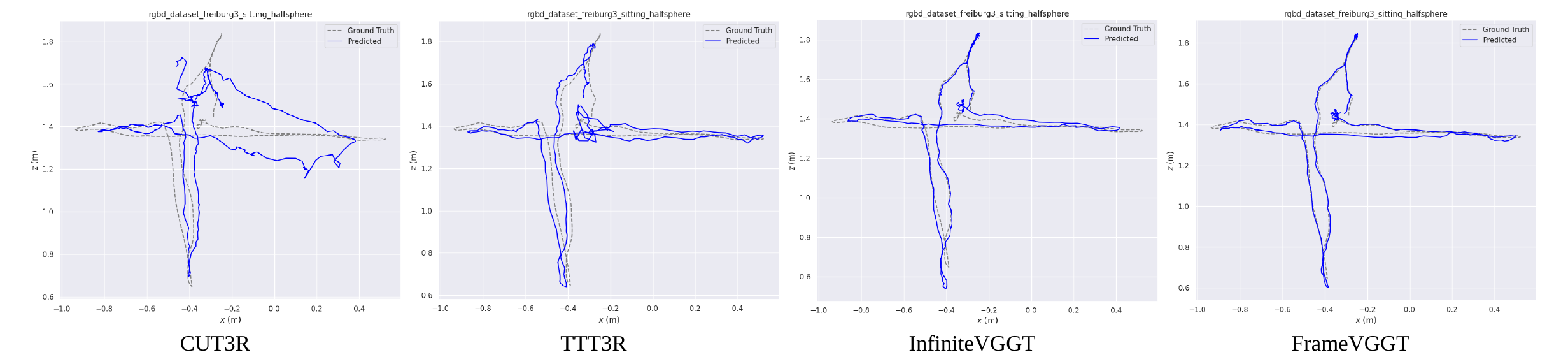}
    \caption{Pose visualization comparison on the TUM dataset. Our method maintains a robustly stable trajectory and ensures geometric integrity even over exceptionally long-duration sequences.}
    \label{fig:tum_pose_vis}
\end{figure}
\begin{table}[t]
\centering
\begin{minipage}[t]{0.48\linewidth}
\centering
\scriptsize
\setlength{\tabcolsep}{5pt}
\renewcommand{\arraystretch}{1.1}
\caption{Pose estimation results on TUM. Best results are shown in bold.}
\label{tab:pose_main}
\begin{tabular}{lccc}
\toprule
Method & ATE$\downarrow$ & RPE$_t\downarrow$ & RPE$_r\downarrow$ \\
\midrule
CUT3R        & 0.1089 & 0.0148 & 0.619 \\
Point3R      & 0.1387 & 0.0369 & 2.347 \\
TTT3R        & 0.0620 & 0.0147 & 0.618 \\
XStreamVGGT  & 0.0728 & 0.0262 &
0.580 \\
InfiniteVGGT & 0.0478 & 0.0138 & 0.350 \\
OVGGT*       & 0.0561 & 0.0215 & 0.448 \\
Ours (12)    & 0.0387 & 0.0140 & 0.346 \\
Ours (16)    & 0.0386 & 0.0137 & 0.342 \\
Ours (20)    & 0.0389 & 0.0135 & 0.339 \\
Ours (24)    & \textbf{0.0385} & \textbf{0.0133} & \textbf{0.336} \\
\bottomrule
\end{tabular}
\end{minipage}\hfill
\begin{minipage}[t]{0.48\linewidth}
\centering
\scriptsize
\setlength{\tabcolsep}{5pt}
\renewcommand{\arraystretch}{1.1}
\caption{Recent ablation on TUM under a fixed total budget ($M_{\text{total}}{=}16$). Recent-$0$ corresponds to Ours ($M{=}16$). Best results are shown in bold.}
\label{tab:pose_slidewindow}
\begin{tabular}{lccc}
\toprule
Sliding Window & ATE$\downarrow$ & RPE$_t\downarrow$ & RPE$_r\downarrow$ \\
\midrule
Recent-$0$ & \textbf{0.0386} & \textbf{0.0137} & 0.342 \\
Recent-$2$ & 0.0399 & 0.0138 & \textbf{0.341} \\
Recent-$4$ & 0.0410 & 0.0140 & 0.342 \\
Recent-$6$ & 0.0432 & 0.0140 & \textbf{0.341} \\
\bottomrule
\end{tabular}
\end{minipage}
\end{table}

\section{Discussion}
\label{sec:discussion}
\subsection{Interpretation of Results}

Across all tasks, FrameVGGT achieves the best accuracy--memory frontier, outperforming token-level bounded-memory baselines with much smaller KV-cache footprints. This suggests that, in bounded streaming geometry, memory organization matters as much as capacity: coherent frame-level context is more effective than fragmented token subsets. Performance is often limited more by complementary geometric coverage than by raw token count, and gains saturate once memory spans the key viewpoints and baselines. By contrast, token-level selection tends to collapse onto a narrow local subset, weakening cross-frame complementarity and causing long-horizon drift, while recent frames alone often provide insufficient geometric baseline. Frame-level retention instead preserves coherent units and broader complementary coverage, yielding more stable long-range performance under fixed memory.

\paragraph{Limitations and Future Work.}
This work studies bounded streaming geometry under a controlled memory design space. 
Our current formulation uses a fixed inference-time policy to isolate the effect of memory organization from training and pipeline changes. 
While the sparse anchor tier provides a lightweight mechanism for long-range reference retention, this study does not exhaust the broader design space of temporal memory allocation. 
Future work may extend the framework with data-driven memory scoring, scene-adaptive allocation, and more flexible long-timescale reference selection to further improve efficiency and stability across diverse streaming scenarios.
\clearpage
{\small
\bibliographystyle{unsrt}
\bibliography{reference}
}
\clearpage


\appendix
\begin{figure*}[t]
    \centering
    \includegraphics[width=\textwidth]{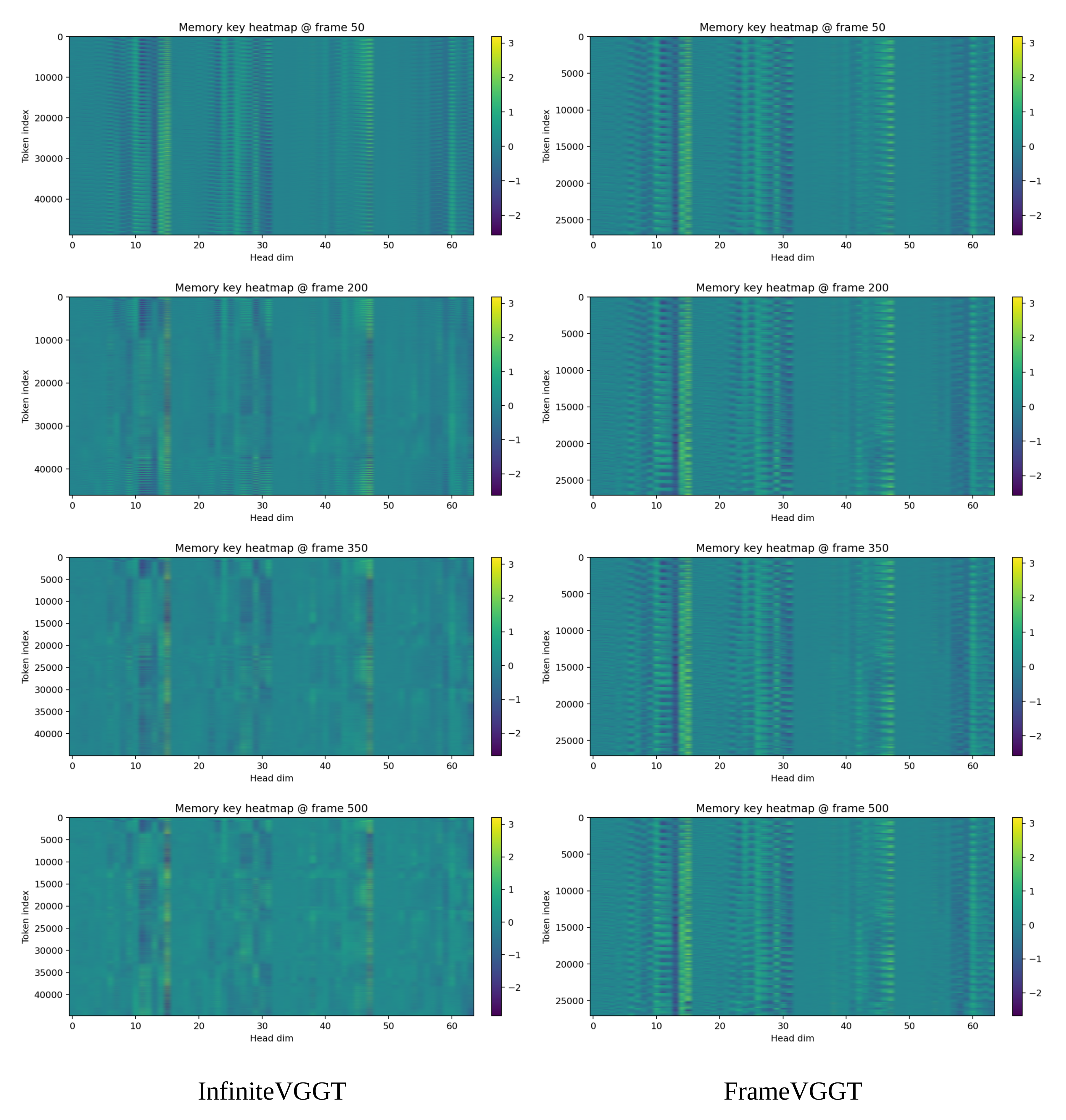}
    \caption{Visualization of memory-key heatmaps at different timesteps. 
        Left: InfiniteVGGT. Right: FrameVGGT.
        Heatmaps visualize saved memory key snapshots at different checkpoints.
        Rows are retained memory tokens (internal memory order) and columns are key dimensions (one layer/head slice; batch=0 and head=0 when applicable).
        Color encodes the key value; the colorbar is auto-scaled per snapshot (ticks chosen automatically), so absolute intensities are not directly comparable across checkpoints.
        Banded patterns indicate groups of tokens with similar key profiles (higher row-wise correlation), while diffuse patterns indicate more heterogeneous keys.}
    \label{fig:memory_key_heatmap}
\end{figure*}
\section{Additional Diagnostics for Token-Level Compression in Geometric Streaming}
\label{app:theory}

This appendix complements Sec.~3 with additional diagnostics for the granularity effects discussed in the main text.
We introduce lightweight proxies that illustrate several degradation patterns observed under bounded token-level retention, while keeping the notation consistent with Sec.~\ref{sec:problem}.

\subsection{Problem Setup}

Consider a video stream of length \(T\), where frame \(t\) produces a set of token vectors
\[
X_t=\{\mathbf{x}_{t,1},\dots,\mathbf{x}_{t,N_t}\},
\qquad
\mathbf{x}_{t,i}\in\mathbb{R}^{d}.
\]
Here, \(T\), \(N_t\), and \(d\) are scalars, \(X_t\) is a set, and \(\mathbf{x}_{t,i}\) is a vector.
Under a bounded-memory policy, a retained subset \(\widetilde{X}_t\subseteq X_t\) is maintained under the global KV-cache budget
\[
\sum_{t=1}^{T} |\widetilde{X}_t| \le M,
\]
where \(M\) is the total memory budget.

As emphasized in the main text, the key question is not only how many entries are retained, but whether the retained subset preserves sufficient \emph{geometrically usable context} for downstream inference.
For geometric prediction, the utility of retained memory depends not only on individual entries, but also on whether they remain organized as coherent observations that jointly support matching, fusion, and long-range consistency.

\paragraph{Geometric-context proxy.}
To make this intuition explicit, we introduce a qualitative scalar proxy \(s(\cdot)\) that reflects the amount of usable geometric context preserved by a retained subset.
Possible instantiations include spatial coverage, reprojection consistency, or other task-dependent measures.
We emphasize that \(s(\cdot)\) is not uniquely defined and is used only as a conceptual diagnostic.

We then define the context-damage proxy
\[
d_t = s(X_t)-s(\widetilde{X}_t),
\]
and the relative context ratio
\[
\rho_t = \frac{s(\widetilde{X}_t)}{s(X_t)}.
\]

These quantities are not used for formal guarantees, but serve as diagnostics of how compression affects geometrically usable context.

\subsection{Context Thinning Under Bounded Memory}

Under a fixed global budget, retained memory must be distributed across an increasing number of frames as the sequence grows.
As a result, the average number of retained observations per frame tends to decrease with longer horizons.

In practice, this leads to progressive \emph{context thinning}: each frame contributes fewer retained entries, reducing the density of usable geometric evidence.
For geometric tasks, where stable estimation relies on redundant multi-view observations, such thinning weakens the effective context available for depth, pose, and reconstruction.

Importantly, degradation is not determined solely by the number of retained entries.
Different retention policies with similar budgets can preserve very different amounts of usable geometric context, depending on how the retained memory is distributed.

\subsection{Context Fragmentation Across Space and Time}

Geometric inference depends not only on how many observations are retained, but also on whether they form coherent multi-view context across frames.

Under token-level retention, selection is performed independently at the token level and is not explicitly constrained to preserve within-frame structure or cross-view compatibility.
As a result, retained entries may no longer correspond to mutually compatible observations of the same underlying structure.

For example, different frames may retain entries from unrelated regions or surfaces, leading to memory that remains populated but lacks stable correspondences.
We refer to this effect as \emph{context fragmentation}, since the retained geometric context becomes scattered across both space and time.

This effect is particularly pronounced in regimes with already limited redundancy, such as low-parallax motion, blur, occlusion, or texture-poor regions.
In such cases, fragmentation further reduces the availability of jointly usable observations and makes inference more sensitive to local inconsistencies.

\subsection{Directional Concentration, Attention Concentration, and Retention Granularity}
\label{app:directional_concentration}

Context thinning and context fragmentation affect not only what is retained, but also how retained memory is used during attention.
For a given layer \(l\), let \(\hat{\mathbf{k}}_i\in\mathbb{R}^{d_l}\) denote an \(\ell_2\)-normalized key vector, where \(d_l\) is the key dimension of layer \(l\).
For a query step \(k\), we partition retained keys into a local subset \(R_k\) and a non-local subset \(N_k\), where \(R_k\) contains keys from a short temporal neighborhood of the query.

We define the local representative direction \(\boldsymbol{\mu}_R\in\mathbb{R}^{d_l}\) by
\[
\boldsymbol{\mu}_R=
\frac{\frac{1}{|R_k|}\sum_{i\in R_k}\hat{\mathbf{k}}_i}
{\left\|\frac{1}{|R_k|}\sum_{i\in R_k}\hat{\mathbf{k}}_i\right\|_2},
\]
and measure directional contrast by
\begin{equation}
\Delta_k
=
\frac{1}{|R_k|}\sum_{i\in R_k}\cos(\hat{\mathbf{k}}_i,\boldsymbol{\mu}_R)
-
\frac{1}{|N_k|}\sum_{j\in N_k}\cos(\hat{\mathbf{k}}_j,\boldsymbol{\mu}_R).
\label{eq:delta_appendix}
\end{equation}
A larger \(\Delta_k\) indicates that keys in the local subset are, on average, more strongly aligned with their dominant direction than the rest of memory.
As in the main text, we use \(\Delta_k\) as a diagnostic of relative directional dominance under bounded memory.

Under standard dot-product attention, attention weights are obtained by applying a softmax over query--key similarity scores.
Let \(\mathbf{q}\in\mathbb{R}^{d_l}\) denote the query vector.
Then
\[
\alpha_i=\frac{\exp(s_i)}{\sum_j \exp(s_j)},
\qquad
s_i=\mathbf{q}^{\top}\hat{\mathbf{k}}_i.
\]
Because the softmax amplifies relative score differences, even moderate directional imbalance can produce increasingly concentrated retrieval over a smaller subset of retained keys.
The following proposition formalizes this effect in a simplified two-group model.

\paragraph{Directional contrast and attention concentration under a two-group model.}
Let the retained memory at query step \(k\) be partitioned into a local subset \(R_k\) and a non-local subset \(N_k\), with cardinalities
\[
m=|R_k|,
\qquad
n=|N_k|.
\]
Assume a simplified score model in which all keys in \(R_k\) share a common similarity score \(a\), and all keys in \(N_k\) share a common similarity score \(b\), with \(a>b\).
The softmax attention weights are
\[
\alpha_i=
\begin{cases}
\dfrac{e^{a}}{m e^{a}+n e^{b}}, & i\in R_k, \\[6pt]
\dfrac{e^{b}}{m e^{a}+n e^{b}}, & i\in N_k.
\end{cases}
\]
Define the score gap
\[
\delta = a-b > 0.
\]

\begin{proposition}
For fixed \(m\) and \(n\), the attention entropy
\[
h_{\mathrm{att}}(\boldsymbol{\alpha})
=
-\sum_i \alpha_i \log \alpha_i
\]
is monotonically decreasing in \(\delta\).
Consequently, the effective memory size
\[
s_{\mathrm{eff}}=\exp\!\bigl(h_{\mathrm{att}}(\boldsymbol{\alpha})\bigr)
\]
is also monotonically decreasing in \(\delta\).
\end{proposition}

\begin{proof}
Let
\[
p_R=\sum_{i\in R_k}\alpha_i
=\frac{m e^{a}}{m e^{a}+n e^{b}}
=\frac{m e^{\delta}}{m e^{\delta}+n},
\qquad
p_N=1-p_R=\frac{n}{m e^{\delta}+n}.
\]
Since all weights are tied within each group, the entropy can be written as
\[
h_{\mathrm{att}}(\boldsymbol{\alpha})
=
- p_R \log \frac{p_R}{m} - p_N \log \frac{p_N}{n}.
\]
Equivalently,
\[
h_{\mathrm{att}}(\boldsymbol{\alpha})
=
h_{\mathrm{bin}}(p_R)+p_R\log m + (1-p_R)\log n,
\]
where
\[
h_{\mathrm{bin}}(p)
=
-p\log p-(1-p)\log(1-p)
\]
denotes the binary entropy.

Now
\[
\frac{d p_R}{d\delta}
=
\frac{m n e^{\delta}}{(m e^{\delta}+n)^2}
>0.
\]
By the chain rule,
\[
\frac{d h_{\mathrm{att}}}{d\delta}
=
\frac{d h_{\mathrm{att}}}{d p_R}\frac{d p_R}{d\delta}.
\]
Using
\[
h_{\mathrm{att}}(p_R)
=
h_{\mathrm{bin}}(p_R)+p_R\log m +(1-p_R)\log n,
\]
we obtain
\[
\frac{d h_{\mathrm{att}}}{d p_R}
=
\log\frac{1-p_R}{p_R}+\log m-\log n
=
\log\frac{m(1-p_R)}{n p_R}.
\]
Substituting
\[
p_R=\frac{m e^{\delta}}{m e^{\delta}+n},
\qquad
1-p_R=\frac{n}{m e^{\delta}+n},
\]
gives
\[
\frac{m(1-p_R)}{n p_R}
=
\frac{m\cdot \frac{n}{m e^{\delta}+n}}{n\cdot \frac{m e^{\delta}}{m e^{\delta}+n}}
=
e^{-\delta}.
\]
Hence
\[
\frac{d h_{\mathrm{att}}}{d p_R}
=
\log(e^{-\delta})
=
-\delta
<0.
\]
Since \(\frac{d p_R}{d\delta}>0\), it follows that
\[
\frac{d h_{\mathrm{att}}}{d\delta}<0.
\]
Therefore \(h_{\mathrm{att}}(\boldsymbol{\alpha})\) decreases monotonically with \(\delta\), and so does \(s_{\mathrm{eff}}=\exp(h_{\mathrm{att}}(\boldsymbol{\alpha}))\).
\end{proof}

This proposition shows that a larger score gap induces more concentrated attention and reduces the effective number of retained keys contributing to retrieval.
Under the additional assumption that the query direction remains positively aligned with \(\boldsymbol{\mu}_R\), a larger \(\Delta_k\) implies that keys in \(R_k\) enjoy a stronger average directional advantage over those in \(N_k\).
In the two-group model above, this corresponds to a larger score gap \(\delta\), and therefore to lower entropy and a smaller effective memory size.
From this perspective, \(\Delta_k\) serves as a useful proxy for the tendency of attention to concentrate on a smaller subset of retained memory.

\paragraph{Retention granularity and local directional dispersion.}
The same diagnostic also clarifies why retention granularity matters.
Define
\[
A_R=
\frac{1}{|R_k|}\sum_{i\in R_k}\cos(\hat{\mathbf{k}}_i,\boldsymbol{\mu}_R),
\qquad
A_N=
\frac{1}{|N_k|}\sum_{j\in N_k}\cos(\hat{\mathbf{k}}_j,\boldsymbol{\mu}_R),
\]
so that
\[
\Delta_k=A_R-A_N.
\]
Because both \(\hat{\mathbf{k}}_i\) and \(\boldsymbol{\mu}_R\) are unit-normalized,
\[
\|\hat{\mathbf{k}}_i-\boldsymbol{\mu}_R\|_2^2
=
2-2\cos(\hat{\mathbf{k}}_i,\boldsymbol{\mu}_R),
\]
and therefore
\begin{equation}
A_R
=
1-\frac{1}{2|R_k|}\sum_{i\in R_k}\|\hat{\mathbf{k}}_i-\boldsymbol{\mu}_R\|_2^2.
\label{eq:AR_identity}
\end{equation}

Now partition the local retained set \(R_k\) into frame groups \(S_t\).
For each frame group \(S_t\), define its mean direction \(\bar{\mathbf{k}}_t\in\mathbb{R}^{d_l}\) by
\[
\bar{\mathbf{k}}_t
=
\frac{1}{|S_t|}\sum_{i\in S_t}\hat{\mathbf{k}}_i,
\]
and its intra-frame directional dispersion by
\[
\sigma_t^2
=
\frac{1}{|S_t|}\sum_{i\in S_t}\|\hat{\mathbf{k}}_i-\bar{\mathbf{k}}_t\|_2^2.
\]

\begin{proposition}
Let \(R_k=\bigsqcup_t S_t\) be a partition of the local retained keys into frame groups, and define
\[
w_t=\frac{|S_t|}{|R_k|}.
\]
Then
\begin{equation}
A_R
\le
1-\frac{1}{2}\sum_t w_t \sigma_t^2.
\label{eq:AR_upper_bound}
\end{equation}
Consequently,
\begin{equation}
\Delta_k
\le
1-\frac{1}{2}\sum_t w_t \sigma_t^2 - A_N.
\label{eq:Delta_upper_bound}
\end{equation}
\end{proposition}

\begin{proof}
Starting from \eqref{eq:AR_identity},
\[
A_R
=
1-\frac{1}{2|R_k|}\sum_{i\in R_k}\|\hat{\mathbf{k}}_i-\boldsymbol{\mu}_R\|_2^2.
\]
We decompose the sum over frame groups:
\[
\frac{1}{|R_k|}\sum_{i\in R_k}\|\hat{\mathbf{k}}_i-\boldsymbol{\mu}_R\|_2^2
=
\frac{1}{|R_k|}\sum_t \sum_{i\in S_t}\|\hat{\mathbf{k}}_i-\boldsymbol{\mu}_R\|_2^2.
\]
For each group \(S_t\), variance decomposition gives
\[
\frac{1}{|S_t|}\sum_{i\in S_t}\|\hat{\mathbf{k}}_i-\boldsymbol{\mu}_R\|_2^2
=
\frac{1}{|S_t|}\sum_{i\in S_t}\|\hat{\mathbf{k}}_i-\bar{\mathbf{k}}_t\|_2^2
+
\|\bar{\mathbf{k}}_t-\boldsymbol{\mu}_R\|_2^2
=
\sigma_t^2+\|\bar{\mathbf{k}}_t-\boldsymbol{\mu}_R\|_2^2,
\]
where the cross term vanishes because
\[
\sum_{i\in S_t}(\hat{\mathbf{k}}_i-\bar{\mathbf{k}}_t)=0.
\]
Hence
\[
\frac{1}{|S_t|}\sum_{i\in S_t}\|\hat{\mathbf{k}}_i-\boldsymbol{\mu}_R\|_2^2
\ge
\sigma_t^2.
\]
Multiplying by \(|S_t|/|R_k|=w_t\) and summing over \(t\) yields
\[
\frac{1}{|R_k|}\sum_{i\in R_k}\|\hat{\mathbf{k}}_i-\boldsymbol{\mu}_R\|_2^2
\ge
\sum_t w_t \sigma_t^2.
\]
Substituting this into \eqref{eq:AR_identity} gives
\[
A_R
\le
1-\frac{1}{2}\sum_t w_t \sigma_t^2,
\]
which proves \eqref{eq:AR_upper_bound}.
Finally, since \(\Delta_k=A_R-A_N\), we immediately obtain
\[
\Delta_k
\le
1-\frac{1}{2}\sum_t w_t \sigma_t^2 - A_N.
\]
\end{proof}

This bound makes the role of retention granularity explicit.
If memory is controlled at excessively fine granularity, token-level selection can retain only a few highly aligned entries from the same local mode, causing the corresponding intra-frame dispersion \(\sigma_t^2\) to collapse toward zero.
In that case, \(A_R\) can approach \(1\), enlarging the directional contrast between \(R_k\) and \(N_k\).
By contrast, frame-level retention preserves each frame's KV contribution as a coherent frame-wise context unit.
Under the mild assumption that retained frame groups exhibit non-degenerate intra-frame directional dispersion, namely
\[
\sigma_t^2 \ge \epsilon > 0,
\]
we obtain
\[
A_R \le 1-\frac{\epsilon}{2},
\qquad
\Delta_k \le 1-\frac{\epsilon}{2}-A_N.
\]
Thus, structured frame-level retention imposes a lower bound on local directional dispersion and reduces the tendency of bounded memory to collapse onto a narrowly dominant local direction.
This is precisely why retaining frame-wise KV segments better preserves complementary geometric context under a fixed memory budget.
\paragraph{Connection to the empirical diagnostic.}
The quantity plotted in Fig.~\ref{fig:mismatch} is exactly the same diagnostic
\(\Delta_k = A_R - A_N\) analyzed above, now measured during streaming inference on the retained memory at each query step.
The propositions in this appendix provide the interpretation of that curve:
a larger \(\Delta_k\) implies a stronger directional advantage of the local retained subset over the non-local subset, which in turn corresponds to more concentrated attention and a smaller effective set of keys contributing to retrieval.
The variance bound further explains why this effect depends on retention granularity:
token-level selection can drive the intra-frame directional dispersion toward zero, allowing \(\Delta_k\) to grow, whereas frame-level retention preserves non-degenerate intra-frame variation and therefore suppresses this collapse.
From this perspective, Fig.~\ref{fig:mismatch} serves as an empirical counterpart to the analysis above, showing that the theoretically motivated diagnostic indeed increases more strongly under token-level bounded memory, while remaining more stable under FrameVGGT.
\subsection{Finite Useful Memory Scale}

These observations also help explain why increasing memory may exhibit diminishing returns.

When memory is very limited, performance is constrained by insufficient support, so increasing capacity improves results.
Beyond a certain scale, however, additional memory tends to include redundant or weakly informative observations, contributing less to downstream inference.

This effect is often more pronounced for token-level policies, which distribute memory across many partially informative fragments.
Frame-level retention instead encourages more structured allocation, and may therefore exhibit smoother scaling behavior under bounded streaming.
\paragraph{Why key space?}
In attention-based memory, Keys and Values play asymmetric roles: Keys determine how memory is addressed by future queries, while Values provide the retrieved content.
Accordingly, redundancy for memory selection is most naturally measured in key space.
By summarizing each frame's incremental KV contribution into a prototype $v_t^{(l)}$, we obtain a compact descriptor of its retrieval direction while preserving block-level coherence.
\begin{figure}[t]
    \centering
    \includegraphics[width=\linewidth]{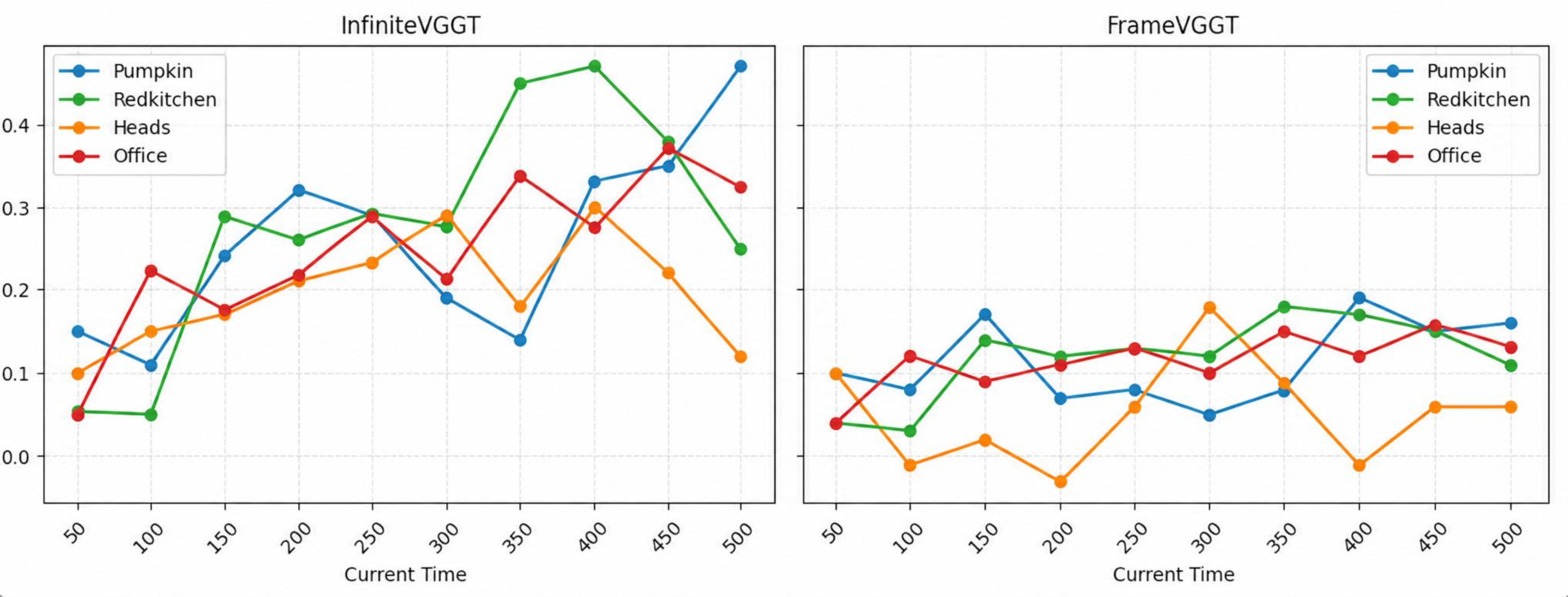}
    \caption{
Temporal evolution of $\Delta_k$ under bounded memory.
Increasing $\Delta_k$ indicates stronger concentration and more brittle fusion.
Token-level retention shows higher and growing $\Delta_k$, while FrameVGGT remains more stable.
    }
    \label{fig:mismatch}
\end{figure}

\section{Anchor Tier: Long-Range Reference Selection}
\label{sec:anchor_tier}

To complement the bounded mid-term bank, we maintain a sparse \emph{anchor tier} that preserves a small number of long-range reference frames.
While the mid-term bank provides the primary bounded memory over recent-to-mid horizons, anchors serve as persistent geometric references when local context becomes unreliable, e.g., under blur, occlusion, weak parallax, or rotation-dominant motion.
This separation allows the model to retain both dense local evidence and sparse global references under a fixed memory budget.

Anchor promotion is regulated by a temporal sparsity constraint
\begin{equation}
g_t = t - t_{\mathrm{last}} \ge G,
\end{equation}
where \(t\), \(t_{\mathrm{last}}\), and \(G\) are scalars.
This prevents over-concentration in locally redundant intervals.
Among eligible candidates, we prioritize frames that are both geometrically reliable and non-redundant.

\paragraph{Frame-level reliability.}
We define the anchor reliability score as
\begin{equation}
\phi(i) = q_i\, s_i,
\end{equation}
where \(q_i \in \mathbb{R}\) denotes a frame-level model-confidence score and \(s_i \in \mathbb{R}\) measures image sharpness.

For \(q_i\), we use a lightweight frame-level confidence proxy derived from the model's own geometric predictions.
Specifically, let \(\{c_{i,u}\}_{u=1}^{N_i}\) denote per-pixel or per-token confidence values at frame \(i\), where \(c_{i,u}\in\mathbb{R}\) and \(N_i\) is the number of pixels or tokens used for aggregation.
Depending on the backbone, these may correspond to matching confidence, visibility confidence, or prediction confidence.
We aggregate them into a single frame-level score by
\begin{equation}
q_i = \frac{1}{N_i}\sum_{u=1}^{N_i} c_{i,u},
\end{equation}
followed by min--max normalization within a sliding temporal buffer.
Intuitively, \(q_i\) favors frames whose predicted geometry is more self-consistent and reliable for later reuse.

For \(s_i\), we use a blur-sensitive image-quality proxy based on the variance of the Laplacian:
\begin{equation}
s_i = \operatorname{Var}\!\left(\nabla^2 I_i\right),
\end{equation}
where \(I_i\) denotes frame \(i\), and \(\nabla^2 I_i\) is its grayscale Laplacian response.
This quantity is high for sharp images with rich local structure and low for blurred or low-detail frames.
We normalize \(s_i\) to \([0,1]\) within the same temporal buffer before combining it with \(q_i\).

\paragraph{Anchor novelty.}
To encourage diversity, we additionally require novelty with respect to the current anchor set \(A_t\).
Let \(\bar{\mathbf{p}}_i \in \mathbb{R}^{d_p}\) denote the normalized pose signature of candidate frame \(i\), where \(d_p\) is the pose-signature dimension.
We define the novelty score as
\begin{equation}
\nu(i) = \min_{a \in A_t} \left(1 - \langle \bar{\mathbf{p}}_i, \bar{\mathbf{p}}_a \rangle \right).
\end{equation}
This discourages promoting frames that are geometrically redundant with already retained anchors.

A candidate frame \(i\) is promoted only if
\begin{equation}
g_t \ge G,
\qquad
\phi(i) \ge \tau_{\mathrm{conf}},
\qquad
\nu(i) \ge \tau_{\mathrm{novel}},
\end{equation}
and the anchor tier retains at most \(M_{\mathrm{anc}}\) frames using FIFO eviction, except for the first frame, which is always preserved as a persistent global reference.

From a retrieval perspective, the anchor tier extends coverage beyond the effective horizon of the mid-term bank by preserving sparse but non-local geometric evidence.
This reduces the tendency of bounded memory to over-concentrate on recent observations and improves robustness over long sequences with little additional memory overhead.

\section{Additional Reconstruction Visualization}
\label{app:more_recon}

\begin{figure}[t]
    \centering
    \includegraphics[width=\linewidth,height=0.86\textheight,keepaspectratio,page=1,trim=0.5cm 0.3cm 0.5cm 0.3cm,clip]{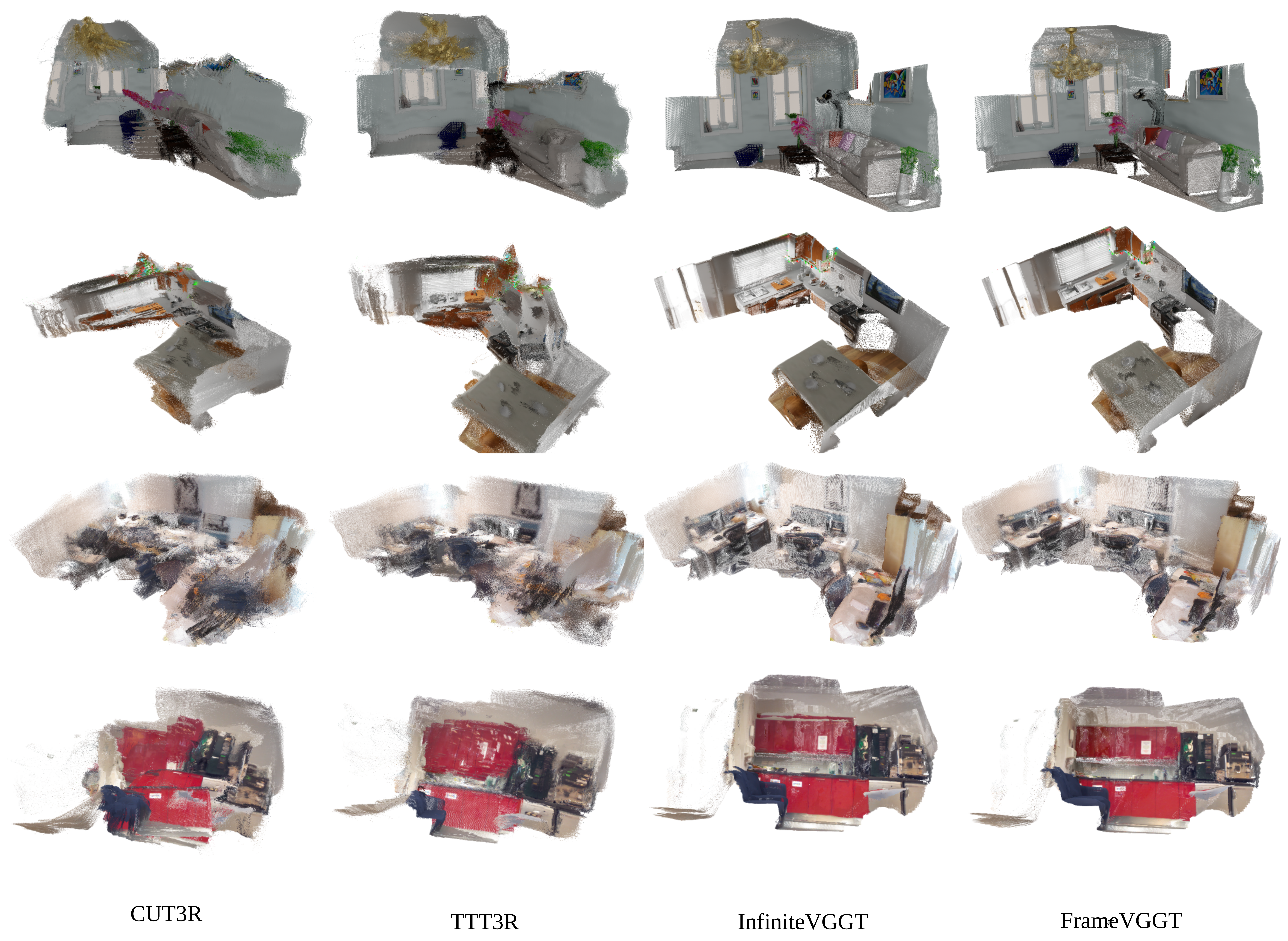}
    \caption{Additional reconstruction visualizations on 7-Scenes and NRGBD.
    We show representative renderings from FrameVGGT over extended streaming sequences.
    These examples complement the quantitative reconstruction results in the main paper by illustrating long-horizon structural consistency under bounded memory.}
    \label{fig:more_recon_vis}
\end{figure}

\begin{figure}[t]
    \centering
    \includegraphics[width=\linewidth,height=0.86\textheight,keepaspectratio,page=1,trim=0.5cm 0.3cm 0.5cm 0.3cm,clip]{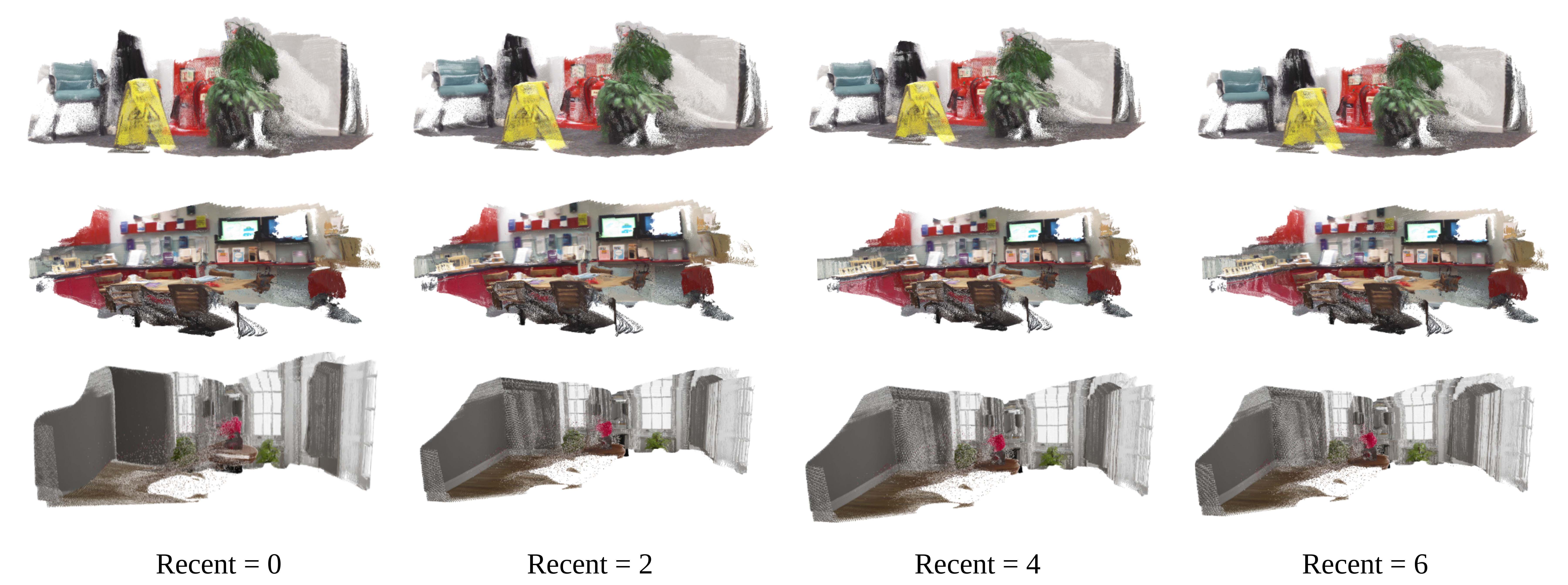}
    \caption{Additional reconstruction visualizations with \textbf{recent-only} memory on 7-Scenes and NRGBD.}
    \label{fig:recon_rec_vis}
\end{figure}

\begin{figure}[t]
    \centering
    \includegraphics[width=\linewidth,height=0.86\textheight,keepaspectratio,page=1,trim=0.5cm 0.3cm 0.5cm 0.3cm,clip]{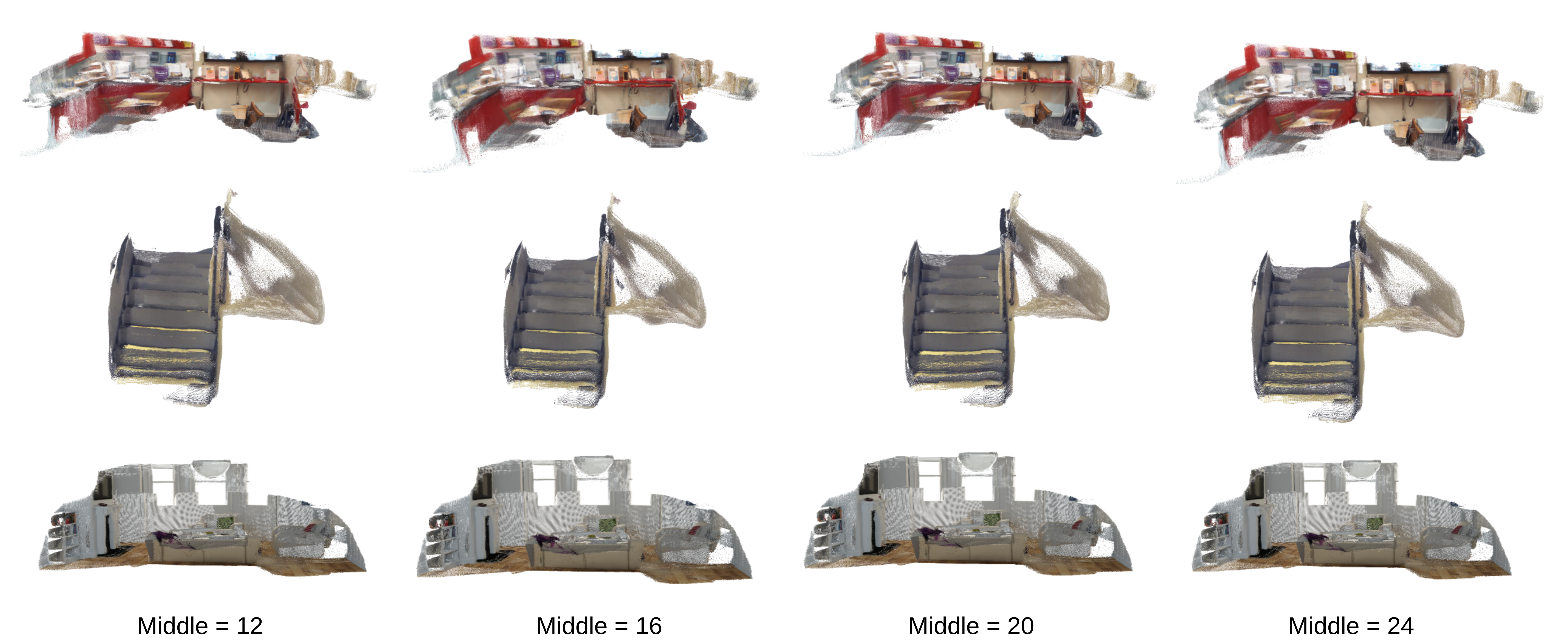}
    \caption{Additional reconstruction visualizations with \textbf{mid-term} memory on 7-Scenes and NRGBD.}
    \label{fig:recon_middle_vis}
\end{figure}

\begin{figure}[t]
    \centering
    \includegraphics[width=\linewidth,height=0.86\textheight,keepaspectratio,page=1,trim=0.5cm 0.3cm 0.5cm 0.3cm,clip]{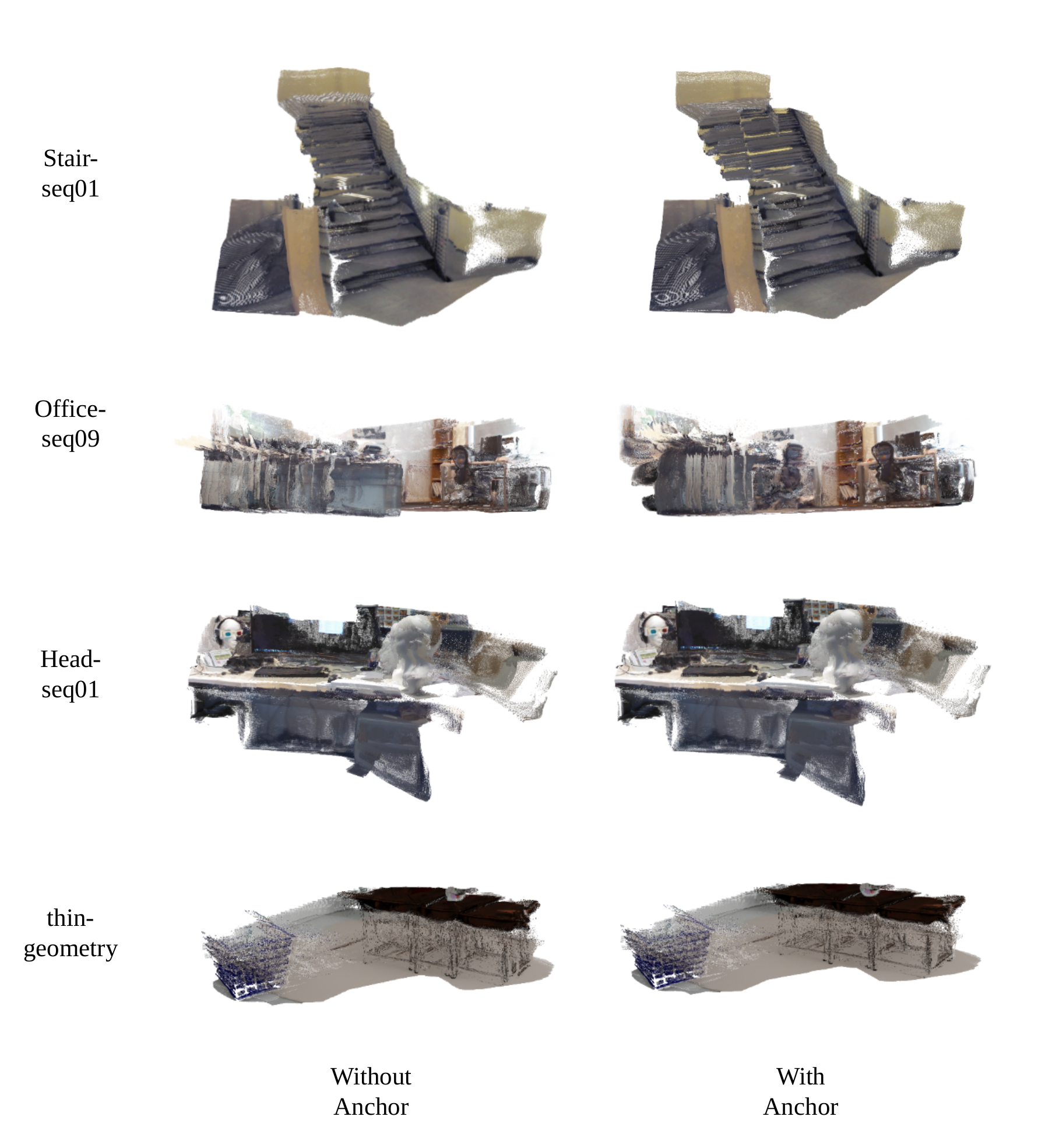}
    \caption{Additional reconstruction visualizations with the \textbf{global anchor tier} on 7-Scenes and NRGBD.}
    \label{fig:recon_anchor_vis}
\end{figure}

\begin{table}[t]
\centering
\footnotesize
\setlength{\tabcolsep}{3pt}
\renewcommand{\arraystretch}{1.08}
\caption{Reconstruction results on 7-Scenes and NRGBD. Best results are shown in bold.}
\label{tab:recon_anchor}
\resizebox{\linewidth}{!}{%
\begin{tabular}{l cccccc|cccccc}
\toprule
{Method}
& \multicolumn{6}{c|}{7-Scenes}
& \multicolumn{6}{c}{NRGBD}
\\
\cmidrule(lr){2-7}\cmidrule(lr){8-13}
& Acc$\downarrow$ & Acc$_{\text{med}}\downarrow$ & Comp$\downarrow$ & Comp$_{\text{med}}\downarrow$ & NC$\uparrow$ & NC$_{\text{med}}\uparrow$
& Acc$\downarrow$ & Acc$_{\text{med}}\downarrow$ & Comp$\downarrow$ & Comp$_{\text{med}}\downarrow$ & NC$\uparrow$ & NC$_{\text{med}}\uparrow$
\\
\midrule
Ours (Without Anchor) & 0.028 & 0.009 & 0.019 & 0.0040 & 0.564 & 0.598 & 0.054 & 0.035 & 0.030 & 0.0060 & 0.670 & 0.782 \\
Ours (With Anchor)    & 0.026 & 0.009 & 0.018 & 0.0040 & 0.571 & 0.607 & 0.053 & 0.034 & 0.031 & 0.0057 & 0.682 & 0.793 \\
\bottomrule
\end{tabular}%
}
\end{table}

Before turning to pose trajectories, we further examine the long-horizon behavior of reconstructed geometry.
Compared with per-frame depth maps, reconstruction renderings provide a more direct view of whether the retained memory preserves \emph{global structural consistency} as evidence accumulates over time.

Fig.~\ref{fig:more_recon_vis} presents additional qualitative reconstruction results on 7-Scenes and NRGBD.
Across extended streaming horizons, the reconstructed geometry remains visually coherent: dominant planes, major scene boundaries, and overall layout structure are stably preserved.
Although regions with weaker multi-view context may exhibit mild local uncertainty as the sequence grows, we do not observe large-scale structural collapse or severe global distortion.
This suggests that the retained memory continues to preserve sufficiently coherent geometric context even in prolonged streaming regimes.

To further isolate the role of memory structure, Figs.~\ref{fig:recon_rec_vis}--\ref{fig:recon_anchor_vis} compare three retention regimes: \textbf{recent-only}, \textbf{mid-term}, and the \textbf{global anchor tier}.

\paragraph{Recent-only memory (Fig.~\ref{fig:recon_rec_vis}).}
When memory is concentrated on only the most recent observations, reconstruction is driven primarily by short-range temporal continuity.
This often suffices in early segments, where adjacent views still provide dense local context.
However, as the horizon extends, the lack of longer-range geometric witnesses makes the reconstruction increasingly vulnerable to accumulated drift, duplication artifacts, and weakened global consistency.
In our framework, this corresponds to a regime where local context remains available but the broader geometric skeleton of the scene is no longer sufficiently preserved.

\paragraph{Mid-term memory (Fig.~\ref{fig:recon_middle_vis}).}
Retaining mid-horizon context substantially improves long-range stability by preserving additional support from intermediate temporal neighborhoods.
Compared with the recent-only regime, dominant structures are more consistently maintained and drift accumulation is visibly reduced.
This indicates that memory aligned with intermediate-range context is more effective at sustaining global reconstruction coherence.
Nevertheless, when the sequence extends well beyond the effective temporal span of the mid-term bank, residual artifacts can still emerge, reflecting the finite coverage of a bounded but non-global memory tier.

\paragraph{Global anchor tier (Fig.~\ref{fig:recon_anchor_vis}).}
Introducing the global anchor tier further strengthens reconstruction stability by preserving sparse but persistent long-horizon reference frames.
Qualitatively, anchors help suppress duplication artifacts, stabilize scene layout, and better maintain a consistent global structure over extended horizons.
In effect, the anchor tier acts as a sparse geometric skeleton that complements the online mid-term pathway with globally distributed reference points.
This qualitative trend is also reflected quantitatively in Tab.~\ref{tab:recon_anchor}: introducing anchors preserves comparable performance on clean or easier sequences while improving robustness on more challenging cases, where long-horizon drift and structural duplication are more likely to accumulate.
Thus, the anchor tier improves hard-case stability without sacrificing performance in well-conditioned streams.

Overall, these additional visualizations suggest that long-horizon reconstruction quality depends not only on the amount of retained evidence, but more fundamentally on the \emph{temporal structure} of that evidence.
Structured retention better preserves the memory organization required for stable geometry, which in turn yields stronger long-range reconstruction consistency.
These observations are consistent with the trajectory behavior analyzed next.

\section{Additional Depth Visualization}
\label{app:more_depth}

\begin{figure}[t]
    \centering
    \includegraphics[width=\linewidth,height=0.82\textheight,keepaspectratio,page=1,trim=0.5cm 0.3cm 0.5cm 0.3cm,clip]{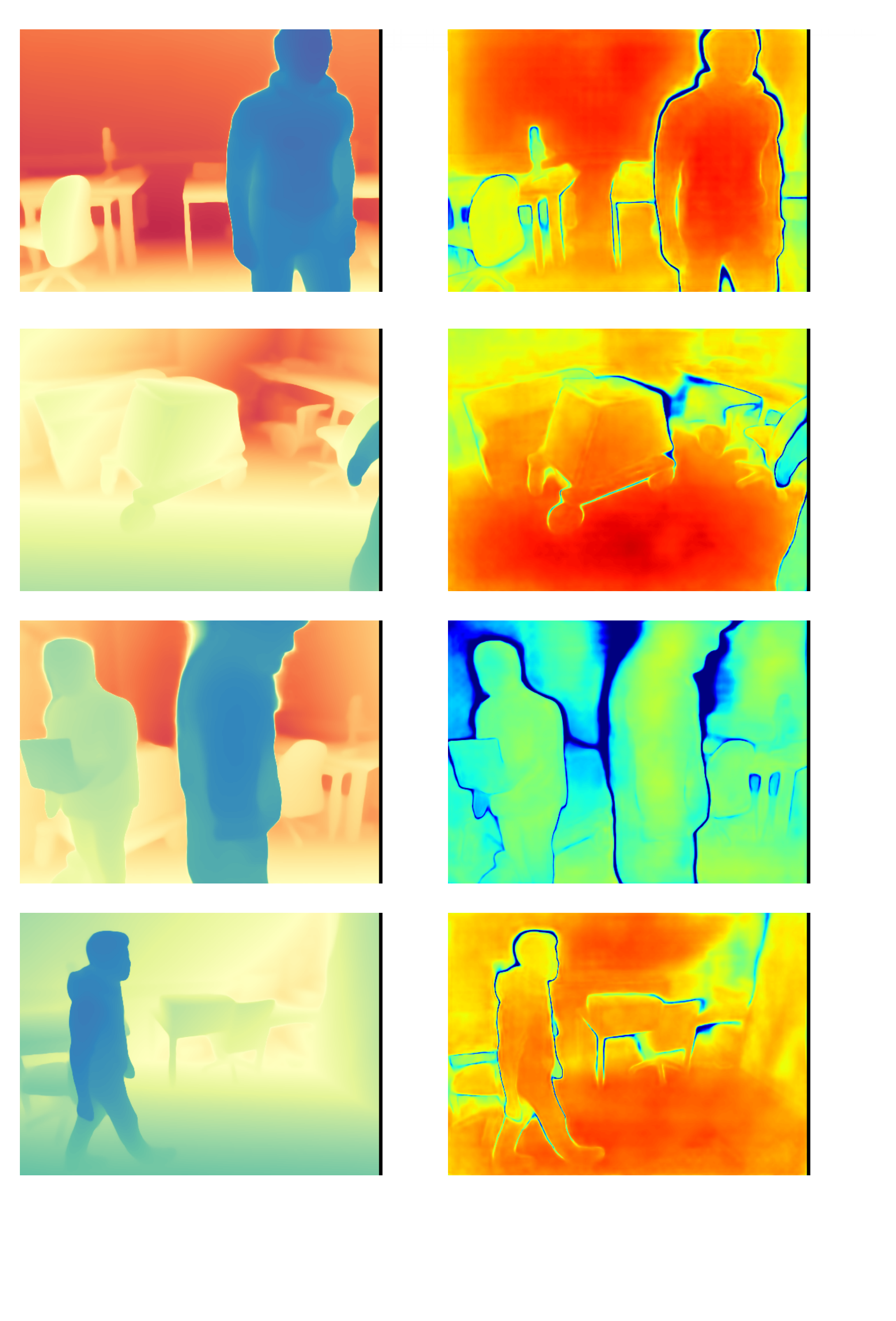}
    \caption{Additional depth visualization results on BONN.
    We visualize predicted depth maps from FrameVGGT over extended streaming sequences.}
    \label{fig:more_depth_vis}
\end{figure}

Before analyzing pose trajectories, we first examine the long-horizon behavior of predicted depth.
Depth stability provides a direct proxy for whether the retained memory continues to preserve sufficiently coherent local geometric context over time.

Unlike pose trajectories, depth visualizations across different methods are not strictly comparable in a fully fair manner because of scale-shift ambiguity, normalization differences, and colormap choices.
We therefore focus on the temporal behavior of FrameVGGT itself rather than presenting side-by-side method comparisons.

Fig.~\ref{fig:more_depth_vis} shows representative depth maps across extended streaming sequences on BONN.
In early segments, the predicted geometry is sharp and spatially coherent.
As time progresses, regions with weaker multi-view contex may exhibit mild local uncertainty; however, large-scale structural collapse, catastrophic distortion, or obvious degeneration of dominant layouts is not observed.
Major planes and scene boundaries remain temporally stable over long horizons, indicating that the retained memory continues to provide sufficiently concentrated memory for depth inference.

These qualitative observations are consistent with the intended role of structured memory retention:
preserving informative mid-horizon context helps maintain local geometric coherence even as the stream becomes progressively longer.
This depth-level stability forms the geometric basis for the trajectory behavior analyzed next.

\section{Additional Pose Visualization}
\label{app:more_pose}

\begin{figure}[t]
    \centering
    \includegraphics[width=\linewidth,height=0.82\textheight,keepaspectratio,page=1,trim=0.5cm 0.3cm 0.5cm 0.3cm,clip]{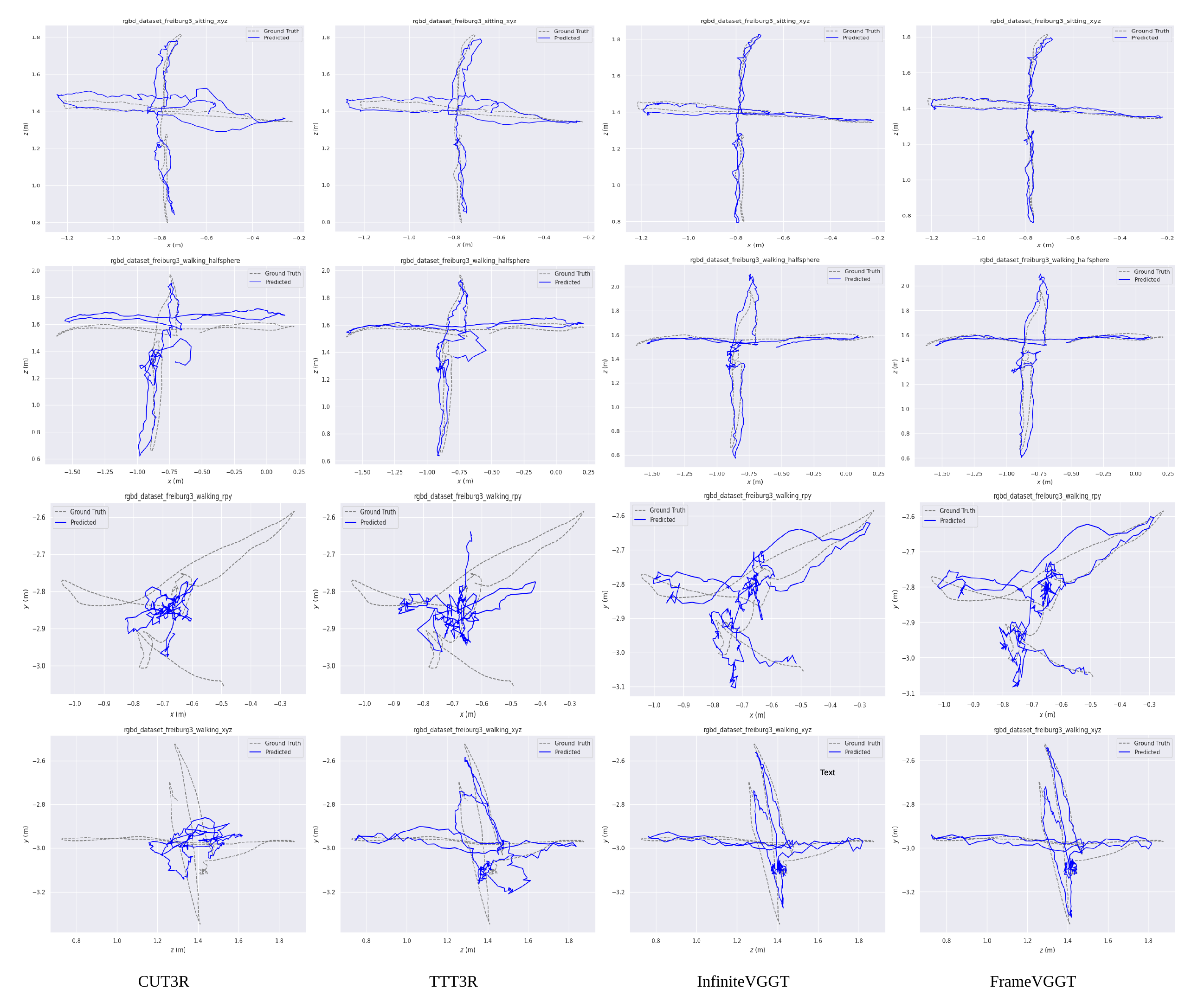}
    \caption{Additional pose visualization results on the TUM dataset. The figure shows representative trajectory comparisons over longer sequences.}
    \label{fig:more_pose_vis}
\end{figure}

To further characterize long-horizon pose behavior beyond aggregate metrics, we provide additional qualitative trajectory visualizations on TUM in Fig.~\ref{fig:more_pose_vis}, together with focused comparisons of different retention regimes in Figs.~\ref{fig:pose_middle_vis} and \ref{fig:pose_rec_vis}.
These figures serve as qualitative evidence of how bounded-memory structure affects long-range trajectory stability.

Fig.~\ref{fig:more_pose_vis} compares representative long-sequence trajectories across methods.
A consistent pattern is that deviations from the global trajectory shape accumulate over time:
many baselines remain reasonably stable in earlier segments but exhibit increasingly visible drift in later portions of the sequence.
Although the precise failure signatures vary---for example, gradual drift versus abrupt deviation---the common trend is the progressive erosion of global consistency under long-horizon bounded streaming.

Figs.~\ref{fig:pose_middle_vis} and \ref{fig:pose_rec_vis} isolate the role of temporal retention structure.
The \emph{mid-term} regime preserves intermediate-range history, which provides complementary constraints once motion extends beyond immediate neighbors.
This additional context helps stabilize the trajectory against drift that cannot be corrected using only near-adjacent evidence.
In contrast, the \emph{recent-only} regime concentrates capacity on short-range observations.
While this often preserves local continuity, it provides weaker constraints on long-horizon global consistency, especially when viewpoint diversity or baseline growth requires support from a broader temporal neighborhood.

Overall, these qualitative examples align with the quantitative results in the main paper:
preserving informative mid-horizon context improves long-range trajectory stability by better maintaining the geometric context structure needed for consistent pose estimation over extended streams.
\section{KV-Cache Budget and Memory Accounting}
\label{app:memory_budget}

To ensure fair comparison across different memory policies, all methods in our experiments operate under comparable KV-cache budget settings.
In particular, memory usage is measured by the effective footprint of the retained Transformer KV cache during streaming inference.

For a Transformer layer $l$, the KV cache size is proportional to
\begin{equation}
\mathrm{Mem}_{KV}^{(l)} \propto H_l \, T_l \, D_l \times 2 ,
\end{equation}
where $H_l$ is the number of attention heads, $T_l$ the number of retained tokens, and $D_l$ the feature dimension per head; the factor of $2$ accounts for keys and values.
Total KV-cache memory is obtained by summing this quantity across all layers.

Although different methods organize memory differently (e.g., token-level vs.\ frame/block-level retention), the reported memory footprint always reflects the actual retained KV tensors used during inference.
Thus, comparisons across methods correspond to comparable KV-cache budgets rather than unequal GPU allocations.

Unless otherwise stated, the reported memory usage excludes unrelated overheads such as dataset loading, visualization buffers, or evaluation pipelines, and reflects only the effective KV-cache footprint of the streaming model.

\begin{figure}[t]
    \centering

    \begin{subfigure}{\linewidth}
        \centering
        \includegraphics[width=\linewidth,height=0.35\textheight,keepaspectratio,page=1,trim=0.5cm 0.3cm 0.5cm 0.3cm,clip]{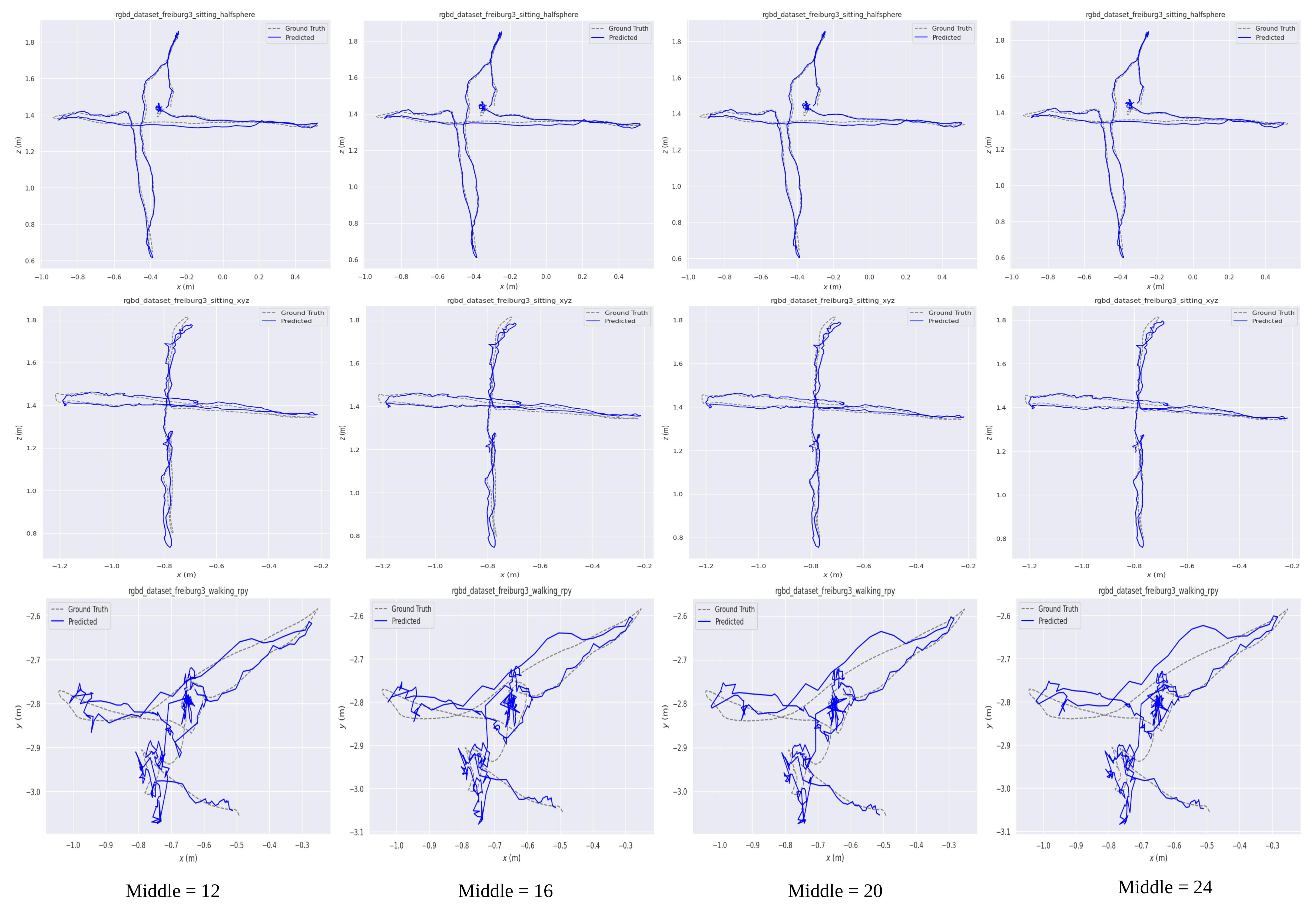}
        \caption{Additional visualization focusing on the behavior of the mid-term memory regime. Preserving intermediate-range history provides stronger temporal context beyond the immediate local context.}
        \label{fig:pose_middle_vis}
    \end{subfigure}

    \vspace{0.8em}

    \begin{subfigure}{\linewidth}
        \centering
        \includegraphics[width=\linewidth,height=0.35\textheight,keepaspectratio,page=1,trim=0.5cm 0.3cm 0.5cm 0.3cm,clip]{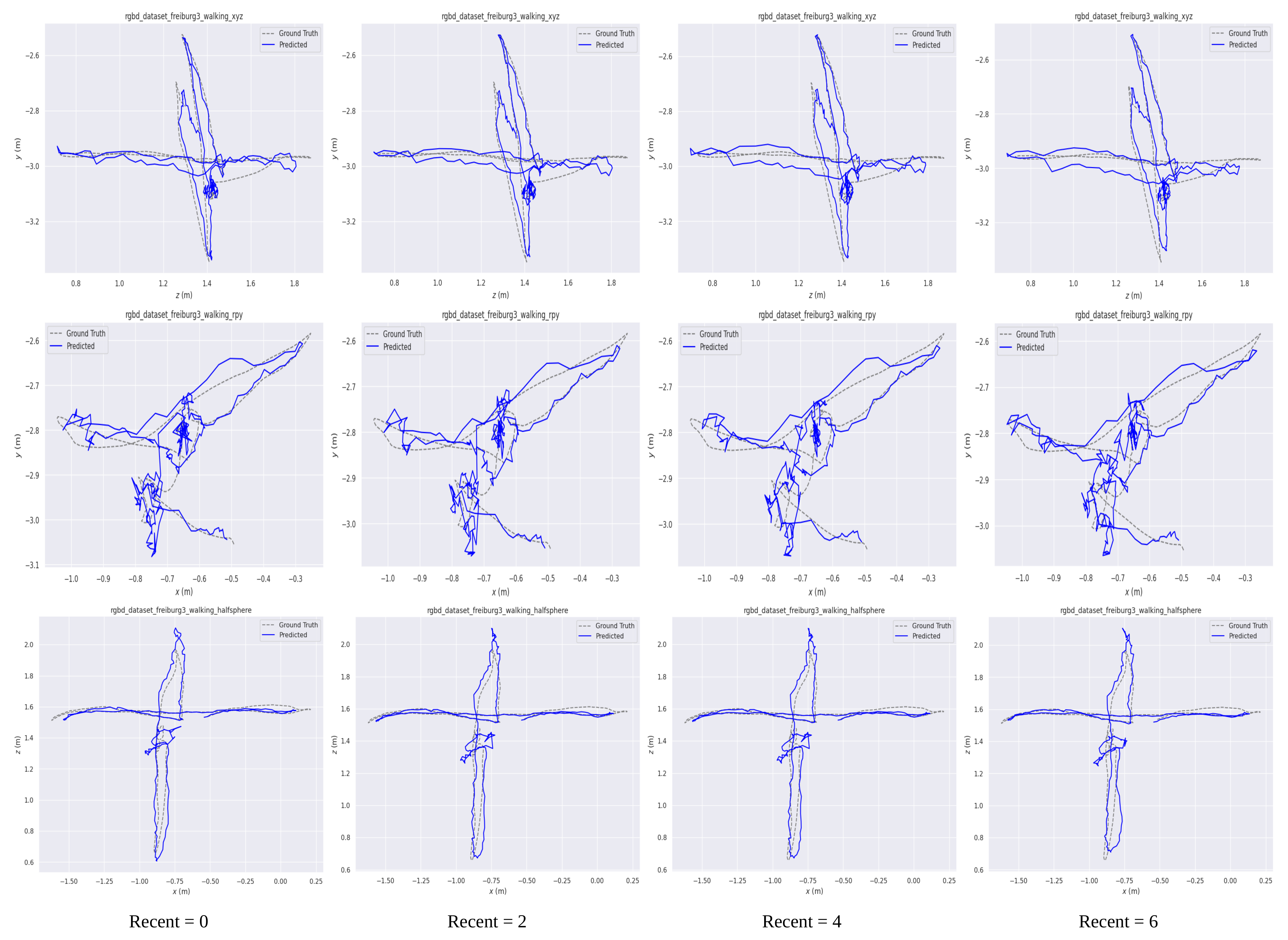}
        \caption{Additional visualization focusing on the behavior of the recent-memory regime. Retaining only recent observations is effective for short-range continuity but offers limited context over extended horizons.}
        \label{fig:pose_rec_vis}
    \end{subfigure}

\end{figure}

\end{document}